 \documentclass[12pt]{alt2021} 

\title[Adaptive Reward-Free Exploration]{Adaptive Reward-Free Exploration}

\usepackage{times}
\usepackage{sty/macros}

\usepackage{todonotes}

\altauthor{
	\Name{Emilie Kaufmann} \Email{emilie.kaufmann@univ-lille.fr}\\
	\addr CNRS \& ULille (CRIStAL), Inria SequeL
	\AND
	\Name{Pierre M\'enard} \Email{pierre.menard@inria.fr}\\
	\addr Inria Lille, SequeL team
	\AND
	\Name{Omar {Darwiche Domingues}} \Email{omar.darwiche-domingues@inria.fr}\\
	\addr Inria Lille, SequeL team
	\AND
	\Name{Anders Jonsson} \Email{anders.jonsson@upf.edu}\\
	\addr Universitat Pompeu Fabra
	\AND
	\Name{Edouard Leurent} \Email{edouard.leurent@inria.fr}\\
	\addr Renault \& Inria Lille, SequeL team
	\AND
	\Name{Michal Valko} \Email{valkom@deepmind.com}\\
	\addr DeepMind Paris
}

\begin{document}

\maketitle

\begin{abstract}
\emph{Reward-free exploration} is a reinforcement learning setting studied by \citet{Jin20RewardFree}, who address it by running several algorithms with regret guarantees in parallel. In our work, we instead  give a more natural \emph{adaptive} approach for reward-free exploration which directly reduces upper bounds on the maximum MDP estimation error. We show that, interestingly, our reward-free UCRL algorithm can be seen as a variant of an algorithm of Fiechter from 1994 \citep{Fiechter94}, originally proposed for a different objective that we call \emph{best-policy identification}. We prove that \OurAlgorithm{} needs of order $({SAH^4}/{\varepsilon^2})(\ln(1/\delta) + S)$ episodes to output, with probability $1-\delta$, an $\varepsilon$-approximation of the optimal policy for \emph{any} reward function. This bound improves over existing sample-complexity bounds in both the small $\varepsilon$ and the small $\delta$ regimes. We further investigate the relative complexities of reward-free exploration and best-policy identification. 
\end{abstract}


\section{Introduction}

Reinforcement learning problems are related to learning and/or acting with a good policy in an unknown, stochastic environment, which requires to perform the right amount of \emph{exploration}. In this work, we consider the discounted episodic setting with discount $\gamma \in (0,1]$ and horizon $H$ and model the environment as a Markov Decision Process (MDP) with finite state space $\cS$ of size $S$ and finite action space $\cA$ of size $A\geq 2$, transition kernels $P = (p_h(\cdot | s,a))_{h,s,a}$ and reward function $r = (r_h(s,a))_{h,s,a}$ for $h \in [H]$\footnote{We use the shorthand $[n] = \{1,\dots,n\}$ for every integer $n \in \N^*$.}, $(s,a) \in \cS \times \cA$. The value of a policy $\pi = (\pi_{1},\dots,\pi_{H})$ in step $h \in [H]$ is given by

\vspace{-0.7cm}

\[V_h^{\pi}(s_h ; r) \triangleq \bE^{\pi}\Big[\Big.\sum_{\ell=h}^H\gamma^{h-1}r_h(s_ {\ell},\pi_{\ell}(s_{\ell})) \Big| s_h\Big].\]
In this definition we explicitly materialize the dependency in the reward function $r$, but the expectation also depends on the transition kernel: for all $\ell\in [H]$, $s_{\ell+1} \sim p_{\ell}(\cdot | s_{\ell},\pi_{\ell}(s_{\ell}))$ and a reward with expectation $r_{\ell}(s_ {\ell},\pi_{\ell}(s_{\ell}))$ is generated. We denote by $\pi^\star$ the optimal policy, such that in every step $h\in [H], V_h^{\pi^\star}(s;r) \geq V_h^{\pi}(s ;r)$ for any policy $\pi$, and by $V^\star$ its value function.

An online reinforcement learning algorithm successively generates trajectories of length $H$ in the MDP, starting from an initial state $s_1$ drawn from some distribution $P_0$. The $t$-th trajectory is generated under a policy $\pi^t$ which may depend on the data collected in the $(t-1)$ previous episodes. Given a fixed reward function $r$, several objective have been considered in the literature: maximizing the total reward accumulated during learning, or minimizing some notion of regret \citep{Azar17UCBVI}, proposing a guess for a good policy after a sufficiently large number of episodes \citep{Fiechter94} or guarantee that the policies used during learning are most of the time $\varepsilon$-optimal \citep{dann15PAC}, see Section~\ref{sec:PACsetup} for a precise description. 

Yet in applications, the reward function $r$ is often handcrafted to incentivize some behavior from the RL agent, and  its design can be hard, so that we may end up successively learning optimal policies for different reward functions. This is the motivation given by \citet{Jin20RewardFree} for the \emph{reward-free exploration problem}, in which the goal is to be able to approximate the optimal policy under \emph{any} reward function after a single exploration phase. More precisely, an algorithm for reward-free exploration should generate a dataset $\cD_N$ of $N$ reward-free trajectories ---with $N$ as small as possible--- such that, letting $\hat{\pi}_{N,r}$ be the optimal policy in the MDP $(\hat{P}_N,r)$ (where $\hat{P}_N$ is the empirical transition matrix based on the trajectories in $\cD_N$),  one has
\begin{equation}\bP\left(\text{ for all reward functions } r, \; \bE_{s_1 \sim P_1}\left[V_1^\star(s; r) - V_1^{\hat{\pi}_{N,r}}(s; r)\right] \leq \varepsilon\right) \geq 1 - \delta.\label{obj:PAC}\end{equation}

The solution proposed by \citet{Jin20RewardFree} builds on an algorithm proposed for the different \emph{regret minimization} objective. In order to generate $\cD_N$, their algorithm first run, for each $(s,h)$, $N_0$ episodes of the \euler algorithm of \citet{Zanette19Euler} for the MDP $(P,r^{(s,h)})$ where $r^{(s,h)}$ is a reward function that gives 1 at step $h$ if state $s$ is visited, and 0 otherwise. For each $(s,h)$, after the corresponding \euler has been executed, the $N_0$ policies used in the $N_0$ episodes of \euler are added to a policy buffer $\Phi$. Once this policy buffer $\Phi$ (which contains $S \times H \times N_0$ policies) is complete , the $\cD_N$ database is obtained by generating $N$ episodes under $N$ policies picked uniformly at random in $\Phi$ (with replacement). \citet{Jin20RewardFree} provide a calibration of $N$ and $N_0$ for which \eqref{obj:PAC} holds, leading to a \emph{sampling complexity}, i.e. a total number of exploration episodes, of $\cO\left(\tfrac{S^2AH^5}{\varepsilon^2}\log\left(\tfrac{SAH}{\delta\varepsilon}\right) + \tfrac{S^4AH^7}{\varepsilon^2}\log^3\left(\tfrac{SAH}{\delta\varepsilon}\right)\right)$.

In this paper, we propose an alternative, more natural approach to reward free exploration, that does not rely on any regret minimizing algorithm. We show that (a variant of) an algorithm proposed by \citet{Fiechter94} for Best Policy Identification (BPI) ---a setting described in details in Section~\ref{sec:PACsetup}--- can be used for reward-free exploration. We give a new, simple, sample complexity analysis for this algorithm which improves over that of \citet{Jin20RewardFree}. This (new) algorithm can be seen as a reward-free variant of UCRL \citep{UCRL10}, and is designed to uniformly reduce the estimation error of the Q-value function of \emph{any} policy under \emph{any} reward function, which is instrumental to prove~\eqref{obj:PAC}, as already noted by \citet{Jin20RewardFree}.

Building on a similar idea, the parallel work of \citet{wang2020on} studies reward-free exploration with a particular linear function approximation, providing an algorithm with a sample complexity of order $d^3 H^6\log(1/\delta)/\epsilon^2$, where $d$ is the dimension of the feature space. In the tabular case, $d=SA$ and the resulting sample complexity becomes worse than the one of \cite{Jin20RewardFree}.  Furthermore, \citet{zhang2020task-agnostic} recently studied a setting in which there are only $N$ possible reward functions in the planning phase, for which they provide an algorithm with complexity $\tcO\left(H^5 S A\log(N)\log(1/\delta)/\epsilon^2\right)$\footnote{The $\tcO$ notation is ignoring logarithmic factors in $1/\varepsilon$ and $\log(1/\delta)$.}.

\paragraph{Alternative views on reward-free RL}

Realistic reinforcement-learning applications often face a challenge of a sparse rewards which at the beginning provides no signal
for decision-making.  Numerous attempts were made to guide the exploration in the beginning,
motivated by curiosity \citep{schmidhuber1991possibility,PMID:22791268}, intrinsic motivation \citep{NIPS2015_5668,NIPS2004_2552}, exploration bonuses
\citep{tang2017exploration, ostrovski2017count} mutual information \citep{montufar2016information} and many of its approximations, for instance with variational autoencoders \citep{NIPS2015_5668}.

Nonetheless, it is even more challenging to analyze the exploration and provide guarantees for it.
A  typical take is to consider a well defined proxy for exploration and analyze that. 
For example,  \citet{lim2012autonomous, gajane2019autonomous} cast the  skill discovery as an ability to reach any state within $L$ hops.
Another example is to look  for policies finding the stochastic shortest path \citep{tarbouriech2019no,cohen2020near}
or aiming for the maximum entropy \citep{hazan2018provably}. In our work, we provide an adaptive counterpart to the work of \citet{Jin20RewardFree} for a \emph{reward-free exploration}.

 \paragraph{Outline}  In Section~\ref{sec:PACsetup}, we present the \emph{reward-free exploration} (RFE) setting and contrast it with other standard PAC reinforcement-learning settings, notably the \emph{best-policy identification} (BPI). The \OurAlgorithm{} algorithm is introduced in Section~\ref{sec:algorithm}. In Section~\ref{sec:analysis}, we present its sample complexity analysis. As a variant of \OurAlgorithm{} was originally proposed by \citet{Fiechter94} for BPI, in Section~\ref{sec:RF_to_BPI}, we investigate the difference in complexity between RFE and BPI, and propose the \OurAlgorithmBPI{} algorithm. Finally, we propose numerical simulations in Section~\ref{sec:experiments} to illustrate how the two algorithms explore, compared to oracle strategies using a generative model.

\section{Several PAC Reinforcement Learning Problems}
\label{sec:PACsetup}

In this section, we formally introduce the reward free exploration problem, which is a particular PAC (Probability Approximately Correct) learning problem. We then contrast it with several other PAC reinforcement learning frameworks that have been studied in the literature.

\paragraph{Reward-free exploration} An algorithm for Reward-Free Exploration (RFE) sequentially collects a database of trajectories in the following way. In each time step $t$, a policy $\pi^{t}=(\pi^{t}_h)_{h=1}^H$ is computed based on data from the $t-1$ previous episodes, a \emph{reward-free episode} $z_{t} = (s_1^{t},a_1^{t},s_2^{t},a_2^{t},\dots,s_H^{t},a_H^{t})$ is generated under the policy $\pi^t$ in the MDP starting from a first state $s_1^{t} \sim P_0$: for all $h\in[H]$, $s_h^t \sim p_h(s_{h-1}^t,\pi^t(s_{h-1}^t)$ and the new trajectory is added to the database: $\cD_{t} = \cD_{t-1} \bigcup \{z_{t}\}$.
At the end of each episode, the algorithm can decide to stop collecting data (we denote by $\tau$ its random stopping time) and outputs the dataset $\cD_{\tau}$. 

A RFE algorithm is therefore made of a triple $((\pi^{t})_{t\in\N},\tau,\cD_{\tau})$. The goal is to build an $(\varepsilon,\delta)$-PAC algorithm according to the following definition, for which the \emph{sample complexity}, that is the number of exploration episodes $\tau$ is as small as possible.

\begin{definition}[PAC algorithm for RFE] \label{def:RFE} An algorithm is $(\varepsilon,\delta)$-PAC for reward-free exploration if
\[\bP\left(\text{for all reward function r}, \left|\bE_{s_1 \sim P_0}\left[V_1^{\star}(s_1;r) - V_1^{\hat{\pi}^\star_{\tau,r}}(s_1 ; r)\right]\right| \leq \varepsilon\right) \geq 1 - \delta,\]
where $\hat{\pi}^\star_{\tau,r}$\footnote{We could also define $\hat{\pi}^\star_{\tau,r}$ to be the outcome of some planning phase that takes as an input $\cD_{\tau}$ and $r$ with controlled planning error. Yet for simplicity we stick to the natural choice of $\hat{\pi}^\star_{\tau,r}$ being the optimal policy in the empirical MDP built from $\cD_{\tau}$, which can be computed exactly using backwards induction in the tabular case that we consider.} is the optimal policy in the MDP parameterized by $(\hat{P}^{\tau}, r)$, with  $\hat{P}^{\tau}$ being the empirical transition kernel estimated from the dataset $\cD_{\tau}$.
\end{definition}

\paragraph{Sample complexity in RL} For the discounted episodic setting that is our focus in this paper, in which learning proceeds by a sequence of episodes, the first formal PAC RL model was proposed by \citet{Fiechter94}. As in this framework a RL algorithm should also \emph{output a guess for a near-optimal policy}, we refer to it as Best Policy Identification (BPI). A BPI algorithm is made of a triple $((\pi^{t})_{t\in\N},\tau,\hat{\pi}_{\tau})$ where $\hat{\pi}_{\tau}$ is the policy returned after $\tau$ steps of exploration.

\begin{definition}[PAC algorithm for BPI] \label{def:BPI} An algorithm is $(\varepsilon,\delta)$-PAC for best policy identification if
\[\bP\left(\bE_{s_1 \sim P_0}\left[V_1^{\star}(s_1) - V_1^{\hat{\pi}_{\tau}}(s_1)\right] \leq \varepsilon\right) \geq 1 - \delta.\]
\end{definition}

In the discounted setting that is the focus of \cite{Fiechter94}, choosing a horizon $H = (1-\gamma)^{-1}\log((\varepsilon(1-\gamma))^{-1})$ the policy $\hat{\pi}_{\tau}$\footnote{This policy is extended to select random actions for $h > H$.} outputted by an $(\varepsilon,\delta)$-PAC algorithm for BPI with horizon $H$ is $2\varepsilon$-optimal in terms of the infinite horizon discounted value function. Yet, this requires an online learning process in which the agent can control the length of episode and use a ``restart button''. This assumption was presented as a limitation in subsequent works, which converged on a different notion of PAC algorithm. While the E$^3$ algorithm of \citet{kearns02E3} stops in some state $s_{\tau}$ and outputs a policy $\hat{\pi}_{\tau}$ that needs to be $\varepsilon$-optimal in that state, other algorithms such as R$_{\text{MAX}}$ \citep{Brafman02RMAX}, Delayed Q-Learning \citep{Strehl06DelayedQL} or MBIE \citep{Strehl08MBIE} do not output a policy, but are proved to be {PAC-MDP} according to a definition formalized by \citet{Kakade03PhD} for discounted or average reward MDPs. Under an $(\varepsilon,\delta)$-PAC MDP algorithm generating a trajectory $(s_t)_{t\in\N}$, there is a polynomial number of time steps $t$ in which $V^{\star}(s_t) - V^{\cA_t}(s_t) > \varepsilon$ where $\cA_t$ is the policy used in the future steps of the algorithms.

The notion of PAC-MDP algorithm was later transposed to the (discounted) episodic setting \citep{dann15PAC,dann2017unifying} as an algorithm such that, with probability $1-\delta$, $\sum_{t=1}^{\infty}\ind(V_1^{\star}(s_1^{t}) - V_1^{\pi^t}(s_1^t) > \varepsilon)$ is upper bounded by a polynomial in $S,A,1/\varepsilon,1/\delta$ and $H$.
%
%
%
PAC-MDP seems to be the most studied PAC reinforcement learning framework these days. However, reward free exploration is closer to the BPI framework: in the latter, an algorithm should stop and output a guess for the optimal policy associated to a particular reward function $r$ (possibly unknown and observed through samples), while in the former it should stop and be able to estimate the optimal policy associated to \emph{any} reward function. In the next section, we show that a variant of the first algorithm proposed by \citet{Fiechter94} for BPI, that we call \OurAlgorithm{} can actually be used for the (harder ?) reward free exploration problem and provide a new sample complexity analysis for it. We discuss further the link between RFE and BPI in Section~\ref{sec:RF_to_BPI}.

Finally, sample complexity results have also been given for reinforcement learning based on a generative model, in which one can build a database of transitions performed in an arbitrary order (without the constrain to generate episodes). In the discounted setting, \citet{Azar12SCGene} propose an improved analysis of Model-Based Q-Value Iteration \citep{KearnsS98MBQVI}, which samples $n$ transitions from every state-action pair and run value-iteration in the estimated MDP. They show that with a total sampling budget $T = nSA = \cO\left((cSA/((1-\gamma)^3\varepsilon^2)) \ln({SA}/{\delta})\right)$, the optimal Q-value in the estimated MDP $\hat{Q}$ satisfies $\| \hat{Q} - Q^\star\|_{\infty} \leq \varepsilon$ with probability larger than $1-\delta$.



\section{Reward-Free UCRL}\label{sec:algorithm}

To ease the presentation of our algorithm, we assume that the first state distribution $P_0$ is supported on a single state $s_1$. Following an observation from \cite{Fiechter94}, this is without loss of generality, as we may otherwise consider an alternative MDP with an extra initial state $s_0$ with a single action $a_0$ that yield a null reward and from which the transitions are $P_0(\cdot |s_0,a_0) = P_{0}$. Indeed, letting $\tilde{V}_0^{\tilde{\pi}}$ denote the value of policy $\tilde{\pi}$ such that $\tilde{\pi}_0 = a_0$ and $\tilde{\pi}_{1:H} = \pi$ for any episodic problem of horizon $H+1$ and discount sequence $(1,1,\gamma,\gamma^{2},\dots)$, it holds that $\tilde V^\star_0(s_0 ; r) - \tilde V^{\tilde \pi}_0(s_0 ; r) = \bE_{
s_1 \sim P_1} \left[V^\star_1(s_1 ; r) - V^{\pi}_0(s_0;r)\right]$.

\paragraph{Notation} For all $h\in [H]$ and $(s,a) \in \cS \times \cA$, we let $n_h^{t}(s,a) = \sum_{i=1}^{t} \ind\left((s_h^i,a_h^i) = (s,a)\right)$ be the number of times the state action-pair $(s,a)$ was visited in step $h$ in the first $t$ episodes and  $n_h^{t}(s,a,s') = \sum_{i=1}^{t} \ind\left((s_h^i,a_h^i,s_{h+1}^i) = (s,a,s')\right)$. This permits to define the empirical transitions
\[\hat p^{t}_h(s' | s,a) = \frac{n_h^t(s,a,s')}{n_h^t(s,a)} \text{ if } n_h^t(s,a)>0, \ \text{ and } \  \hat p^{t}_h(s' | s,a) = \frac{1}{S} \text{ otherwise}. \]
We denote by $\hat{V}^{t,\pi}_h(s ; r)$ (resp. $\hat{Q}^{t,\pi}_h(s ,a; r)$) the value (resp.\, Q-values) functions in the empirical MDP with transition kernels $\hat{P}^t$ and reward function $r$, where we recall that the Q-value of a policy $\pi$ in a MDP with transitions $p_h(s' |s,a)$ and mean reward $r_h(s,a)$ is defined by
$Q_h^{\pi}(s,a ; r) = r_h(s,a) + \gamma \sum_{s' \in \cS} p_h(s' |s,a) V_{h+1}^{\pi}(s')$.
Finally, we let $\sigma_h = \sum_{i=0}^{h-1}\gamma^i$ and note that $\sigma_h \leq h$.

\paragraph{Error upper bounds} \OurAlgorithm{} is based on an upper bound on the estimation error for each policy $\pi$ (and each value function $r$). For every $\pi$, $r$, $t$, we define this error as
\[\hat{e}_h^{t,\pi}(s,a ; r) := |\hat{Q}^{t,\pi}_h(s,a ; r) - Q^{\pi}_h(s,a ; r)|.\]
The algorithm relies on an ``upper confidence bound'' $\bar{E}_h^{t}(s,a)$ for the error defined recursively as follows: 
$\bar{E}_{H+1}^{t}(s,a)  =  0$ for all $(s,a)$ and, for all $h \in [H]$, with the convention $1/0 = +\infty$,
\begin{equation}
 \bar{E}_{h}^{t}(s,a)  =\min\!\left(\!\gamma\sigma_{H-h},\,\gamma\sigma_{H-h}\sqrt{\frac{2\beta(n_h^t(s,a),\delta)}{n_h^t(s,a)}} + \gamma \sum_{s'} \hat p_h^t(s' |s,a) \max_{b}\bar{E}_{h+1}^{t}(s',b)\right) \label{def:bound}\,. \end{equation}
Although $\bar{E}_{h}^{t}(s,a)$ does not depend on a policy $\pi$ or a reward function $r$, Lemma~\ref{lem:optimism} shows that it is a high-probability upper bound an the error $\hat{e}_h^{t,\pi}(s,a ; r)$ for \emph{any $\pi$ and $r$}.

 \begin{lemma}\label{lem:optimism} With $\KL(p || q) = \sum_{s \in \cS} p(s) \ln \tfrac{p(s)}{q(s)}$ the Kullback-Leibler divergence between two distributions over $\cS$, on the event 
  \[\cE = \left\{\forall t \in \N, \forall h \in [H], \forall (s,a), \KL\big(\hat{p}^t_h(\cdot | (s,a)), p_h(\cdot | (s,a))\big)\leq \tfrac{\beta(n_h^t(s,a),\delta)}{n_h^t(s,a)}\right\}\;,\]
it holds that for any policy $\pi$ and reward function $r$, $\hat{e}_h^{t,\pi}(s,a ; r) \leq \bar{E}_{h}^{t}(s,a)$.
  \end{lemma}

\begin{proof} From the Bellman equations in the empirical MDP and the true MDP, 
\begin{eqnarray*}
 \hat{Q}_h^{t,\pi}(s,a; r) & = & r_h(s,a) + \gamma \sum_{s'} \hat{p}_h^{t}(s' | s,a)  \hat{Q}_{h+1}^{t,\pi}(s',\pi(s'); r)\\
\text{and } \ {Q}_h^{\pi}(s,a ; r) & = & r_h(s,a) + \gamma \sum_{s'} {p}_h(s' | s,a)  {Q}_{h+1}^{\pi}(s',\pi(s'); r)\;.
\end{eqnarray*}
Hence
\begin{eqnarray*}
\hat{Q}_h^{t,\pi}(s,a ; r) - {Q}_h^{\pi}(s,a ; r) &=& \gamma \sum_{s'} \left(\hat p_h^t(s'|s,a) - p_h(s'|s,a)\right){Q}_{h+1}^{\pi}(s',\pi(s'); r) \\ && + \gamma\sum_{s'} \hat{p}_h^{t}(s' |s,a) \left( \hat{Q}_{h+1}^{t,\pi}(s',\pi(s'); r) - {Q}_{h+1}^{\pi}(s',\pi(s') ; r)\right).
\end{eqnarray*}
It follows that, for $n_h^t(s,a)>0$, using successively that ${Q}_{h+1}^{\pi}(s',a'; r)\leq \sigma_{H-h}$, the definition of event $\cE$ and Pinsker's inequality,
\begin{align*}
  \hat{e}_h^{t,\pi} (s,a ; r) &\leq \gamma \sum_{s'} \left| \hat{p}_h^t(s'|s,a) - p_h(s'|s,a)\right|{Q}_{h+1}^{\pi}(s',\pi(s'); r) \\
   & \qquad\qquad + \gamma\sum_{s'} \hat{p}_h^{t}(s' |s,a)  \left| \hat Q_{h+1}^{t,\pi}(s',\pi(s') ; r) - Q^{\pi}_{h+1}(s',\pi(s') ; r)\right| \\
& \leq  \gamma \sigma_{H-h}  \left\| \hat{p}_h^t(\cdot|s,a) - p_h(\cdot|s,a)\right\|_1 + \gamma\sum_{s'} \hat{p}_h^{t}(s' |s,a) \hat{e}_{h+1}^{t,\pi} (s',\pi(s') ; r)\\
&\leq   \gamma\sigma_{H-h}\sqrt{\frac{2\beta(n_h^t(s,a),\delta)}{n_h^t(s,a)}} +\gamma\sum_{s'} \hat{p}_h^{t}(s' |s,a) \hat{e}_{h+1}^{t,\pi} (s',\pi(s') ; r)\,.
\end{align*}
Then, noting that $\hat{e}_h^{t,\pi} (s,a ; r)\leq \gamma \sigma_{H-h}$, it holds for all $n_h^t(s,a)\geq 0$,
\begin{equation}
  \label{eq:upper_bound_he}
\hat{e}_h^{t,\pi} (s,a ; r)\! \leq \min\!\left(\!\gamma\sigma_{H-h},\gamma\sigma_{H-h}\sqrt{\frac{2\beta(n_h^t(s,a),\delta)}{n_h^t(s,a)}} +\gamma\!\sum_{s'}\! \hat{p}_h^{t}(s' |s,a) \hat{e}_{h+1}^{t,\pi} (s',\pi(s') ; r)\!\right)\!.
\end{equation}

We can now prove the result by induction on $h$. The base case for $H+1$ is trivially true since $\hat{e}_{H+1}^{t,\pi} (s,a ; r)=\bar{E}_{H+1}^t(s,a)=0$ for all $(s,a)$. Assume the result true for step $h+1$, using~\eqref{eq:upper_bound_he} we get for all $(s,a)$,
\begin{align*}
  \hat{e}_h^{t,\pi} (s,a ; r) &\leq\min\!\!\left(\gamma\sigma_{H-h},\, \gamma\sigma_{H-h}\sqrt{\frac{2\beta(n_h^t(s,a),\delta)}{n_h^t(s,a)}} +\gamma\sum_{s'} \hat{p}_h^{t}(s' |s,a) \hat{e}_{h+1}^{t,\pi} (s',\pi(s') ; r)\right)\\
  &\leq \min\!\!\left(\gamma\sigma_{H-h},\, \gamma\sigma_{H-h}\sqrt{\frac{2\beta(n_h^t(s,a),\delta)}{n_h^t(s,a)}} +\gamma\sum_{s'} \hat{p}_h^{t}(s' |s,a) \max_{b\in\cA}\bar{E}^t_{h+1}(s',b)\right)\\
  &= \bar{E}^t_{h+1}(s,a)\,.
\end{align*}

\end{proof}

\paragraph{Sampling rule and stopping rule} The idea of \OurAlgorithm{} is to uniformly reduce the estimation error all policies under all possible reward functions by being greedy with respect to the upper bounds $\bar{E}^{t}$ on these errors.  \OurAlgorithm{} stops when the error in step $1$ is smaller than $\varepsilon/2$:
\begin{itemize}
 \item \textbf{sampling rule}: the policy $\pi^{t+1}$ is the greedy policy with respect to $\bar E^{t}(s,a)$, that is
  \[\forall s \in \cS, \forall h \in [h], \ \ \pi^{t+1}_h(s) = \text{argmax}_a \bar E_h^{t}(s,a).\]
 \item \textbf{stopping rule}: $\tau = \inf \left\{ t\in \N : \bar E_h^{t}(s_1,\pi_1^{t+1}(s_1)) \leq \varepsilon/2 \right\}$.
\end{itemize}

This algorithm is very similar to the one originally proposed by \citet{Fiechter94} for Best Policy Identification in the discounted case. The main difference is that the original algorithm additionally uses some scaling and rounding: the index used are integers, defined as $\tilde E_h^{t}(s,a) = \lfloor \bar E_{h}^{t}(s,a) / \eta \rfloor$ for some parameter $\eta>0$, and the algorithm stops when $\tilde E_h^{t}(s,a)$ is smaller than a slightly different threshold. The reason for this discretization is the use of a combinatorial argument in the sample complexity analysis, which says that every $m$ time steps (with $m$ that is a function of $S,A,\delta$ and $\varepsilon$), at least one of the indices must decrease. The other difference is that the term $\sigma_{H-h} \sqrt{2{\beta(n_h^{t}(s,a),\delta)}/{n_h^ {t}(s,a)}}$ in \eqref{def:bound} is replaced by $\sigma_{H-h}\sqrt{2\log\left({2SAH}/{\delta}\right)}$, which we believe is not enough guarantee the corresponding index to be a high-probability upper bounds on $\hat{e}^{t,\pi}(s,a ; r)$\footnote{There are some missing union bounds in the concentration argument given by \citet{Fiechter94}.}.

In the next section, we propose a different analysis for \OurAlgorithm{} compared to the original analysis of \cite{Fiechter94}, which yields an improved sample complexity in the more general discounted episodic setting. 

\section{Theoretical Guarantees for \OurAlgorithm{}}\label{sec:analysis}

We show that \OurAlgorithm{} is $(\varepsilon,\delta)$-PAC for reward-free exploration and provide a high-probability upper bound on its sample complexity.

\subsection{Correctness and Sample Complexity}

First, for every reward function $r$, one can easily show (see Appendix~\ref{proof:inclusion}) that for all $t\in \N^*$, 
\begin{equation}\left\{ \forall \pi, \left|\hat V_1^{t,\pi}(s_1; r) - V_1^{\pi}(s_1; r) \right| \leq \varepsilon/2\right\} \subseteq \left\{ V_1^{\star}(s_1;r) - V_1^{\hat{\pi}^\star_{t,r}}(s_1;r) \leq \varepsilon\right\}.\label{estimation_to_value}\end{equation}
This property is already used in Lemma 3.6 of \cite{Jin20RewardFree}, where an extra planning error is allowed, whereas we assume that the optimal policy $\hat{\pi}^\star_{t,r}$ in $(\hat{P}^{t},r)$ is computed exactly (with backward induction). Hence, a sufficient condition to prove the correctness of \OurAlgorithm{} is to establish that, when it stops, the estimation errors for all policies \emph{and all reward functions} is smaller than $\varepsilon/2$. But the stopping rule of \OurAlgorithm{} is precisely designed to achieve this property.

\begin{lemma}[correctness] \label{lem:correctness} On the event $\cE$, for any reward function $r$, $   V_1^{\star}(s_1) - V_1^{\hat{\pi}^\star_{\tau,r}}(s_1) \leq \varepsilon$.
\end{lemma}

\begin{proof} By definition of the stopping rule, $\bar{E}_1^{\tau}(s_1,\pi_1^{\tau+1}(s_1)) \leq \varepsilon/2$. As $\pi^{\tau+1}$ is the greedy policy w.r.t. $\bar E^{\tau}$, this implies that for all $a \in \cA$, $\bar{E}_1^{\tau}(s_1,a) \leq \varepsilon/2$. Hence, by  Lemma~\ref{lem:optimism} on the event $\cE$, for all policy $\pi$,  all reward function $r$, and all action $a$,  $\hat{e}_1^{\tau,\pi}(s_1,a ; r) \leq \varepsilon/2$. In particular, for all $\pi$ and $r$,  $|\hat{V}_1^{\tau,\pi}(s_1 ; r) - {V}_1^{\pi}(s_1 ; r)| \leq \varepsilon/2$, and the conclusion follows from the implication \eqref{estimation_to_value}.
\end{proof}

We now state our main results for \OurAlgorithm{}. We prove that for a well-chosen calibration of the threshold $\beta(n,\delta)$, the algorithm is $(\varepsilon,\delta)$-PAC for reward-free exploration and we provide a high-probability upper bound on its sample complexity.

\begin{theorem}\label{thm:sc} \OurAlgorithm{} using threshold $\beta(n,\delta) = \log\!\big(2SAH/\delta\big) + (S-1)\log\big(e(1+n/(S-1))\big)$
is $(\varepsilon,\delta)$-PAC for reward-free exploration. Moreover, with probability $1-\delta$, 
{\footnotesize
\begin{align*}
\tau \leq \!\frac{\cC_{H}SA}{\varepsilon^2}
	\!\spa{
		\log\!\pa{\! \frac{2SAH}{\delta}\! }
		\!+ 2(S\!-\!1)\!\log\!\pa{\!
			\frac{\cC_{H}SA}{\varepsilon^2}\!
			\pa{\!
				\log\!\pa{\! \frac{2SAH}{\delta} \!} \! + (S\!-\!1)\pa{\!\sqrt{e} + \sqrt{\frac{e}{S-1}}}
			\!}\!
		} \!+ (S\!-\!1)
	}
\end{align*}
}
where  $\cC_{H} = 144(1+\sqrt{2})^2\sigma_H^4$.
\end{theorem}

From Theorem~\ref{thm:sc}, the number of episodes of exploration needed is of order
\[\frac{SAH^4}{\varepsilon^2}\ln\left(\frac{2SAH}{\delta}\right) +  \frac{S^2AH^4}{\varepsilon^2}\ln\left(\frac{SAH^4}{\varepsilon^2}\ln\left(\frac{2SAH}{\delta}\right)\right),\]
up to (absolute) multiplicative constants. As explained in Appendix~\ref{app:stationary}, for stationary transitions, we can further replace $H^4$ by $H^3$ in this bound. We now examine the scaling of this bound when $\varepsilon$ goes to zero and when $\delta$ goes to zero. In a regime of small $\varepsilon$, our $\tilde{\cO}\left({S^2AH^4}/{\varepsilon^2}\right)$ bound improves the dependency in $H$ compared to the one given by \cite{Jin20RewardFree} from $H^5$ to $H^4$ (and to $H^3$ for stationary transitions). This new bound is matching the lower bound of \cite{Jin20RewardFree} up to a factor $H^2$ (and a factor $H$ for stationary transitions). Then, in a regime small $\delta$, the sample complexity of \OurAlgorithm{} scales in $\cO\left(({SAH^4}/{\varepsilon^2})\ln\left({1}/{\delta}\right)\right)$, which greatly improves over the $\cO(({S^4AH^7}/{\varepsilon})\left(\ln\left({1}/{\delta}\right)\right)^3)$ scaling of the algorithm of \cite{Jin20RewardFree}. Finally, we note that our result also improves over the original sample complexity bound given by \cite{Fiechter94}, which is in  $\tilde\cO\left(({S^2A}/(1 - \gamma)^7\varepsilon^3)) \log\left({SA}/((1-\gamma)\delta)\right)\right)$ in the discounted setting (for which $H \sim 1/(1-\gamma)$).

\subsection{Proof of Theorem~\ref{thm:sc}}

We first introduce a few notation. We let $p_h^{\pi}(s,a)$ be the probability that the state action pair $(s,a)$ is reached in the $h$-th step of a trajectory generated under the policy $\pi$, 
and we use the shorthand $p_h^t(s,a)= p_h^{\pi^t}(s,a)$.
We introduce the \emph{pseudo-counts} $\bar n_h^{t}(s,a) = \sum_{i=1}^{t} p_h^i(s,a)$ and define
\[
\cE^{\text{cnt}} =\left\{ \forall t \in \N^\star, \forall h\in [H],\forall (s ,a)\in\cS\times\cA:\ n_h^t(s,a) \geq \frac{1}{2}\bar n_h^t(s,a)-\beta^{\text{cnt}}(\delta)  \right\}\,,
\]
where $\beta^{\text{cnt}}(\delta) = \log\big(2SAH/\delta\big)$. Recalling the event $\cE$ defined in Lemma~\ref{lem:optimism}, we let $\cF = \cE \cap \cE^{\text{cnt}}$. Lemma~\ref{lem:concentration} in Appendix~\ref{app:concentration_events} shows that $\bP(\cE) \geq 1 - \delta /2$ and $\bP(\cE^{\text{cnt}}) \geq 1- \delta /2$, which yields $\bP(\cF) \geq 1 -\delta$. From Lemma~\ref{lem:correctness}, on the event $\cF$,  it holds that $ V_1^{\star}(s_1) - V_1^{\hat{\pi}^\star_{\tau,r}}(s_1) \leq \varepsilon$ for all reward function $r$, which proves that \OurAlgorithm{} is $(\varepsilon,\delta)$-PAC.

We now upper bound the sample complexity of \OurAlgorithm{} on the event $\cF$, postponing the proof of some intermediate lemmas to Appendix~\ref{app:main_thm}. The first step is to introduce an average upper bound on the error at step $h$ under policy $\pi^{t+1}$ defined as
\[q_{h}^{t} = \sum_{(s,a)} p_{h}^{t+1}(s,a) \bar{E}_h^{t}(s,a).\]
The following crucial lemma permits to relate the errors at step $h$ to that at step $h+1$.

\begin{lemma}\label{lem:induction} On the event $\cE$, for all $h \in [H]$ and $(s,a) \in \cS \times \cA$,
 \[\bar{E}_h^{t}(s,a) \leq 3\sigma_{H-h}\left[\sqrt{\frac{\beta(n_h^t(s,a),\delta)}{n_h^t(s,a)}} \wedge 1 \right] + \gamma\sum_{s'\in \cS} p_{h}(s'| s,a)\bar{E}_{h+1}^{t}(s',\pi^{t+1}(s'))\;.\]
\end{lemma}

Thanks to Lemma~\ref{lem:induction}, the average errors can in turn be related as follows: {\small
\begin{eqnarray} q_h^t & \leq & \!\!3\sigma_{H-h} \!\!\sum_{(s,a)}\! p_{h}^{t+1}(s,a)\!\left[\sqrt{\frac{\beta(n_h^t(s,a),\delta)}{n_h^t(s,a)}}\wedge 1 \right]\! + \gamma \!\sum_{(s,a)}\!\sum_{(s',a')}\!\!p_h^{t+1}(s,a) p_h(s' |s,a)\ind(a'\! = \! \pi^{t+1}\!(s'))  \overline{E}_{h+1}^{t}(s',a')\nonumber \\
 &\leq& \!\!3\sigma_{H-h} \!\!\sum_{(s,a)}\! p_{h}^{t+1}(s,a)\!\left[\sqrt{\frac{\beta(n_h^t(s,a),\delta)}{n_h^t(s,a)}}\wedge 1 \right]\! + \gamma q_{h+1}^{t}\;.\label{crucial}\end{eqnarray}}
 For $h=1$, observe that $p_{h}^{t+1}(s_1,a) \bar{E}_h^{t}(s_1,a) = \bar{E}_h^{t}(s_1,\pi_1^{t+1}(s_1)) \ind\left(\pi_1^{t+1}(s_1) = a\right)$, as the policy is deterministic. Now, if $t < \tau$,  $\bar{E}_h^{t}(s_1,\pi_1^{t+1}(s_1) \geq \varepsilon/2$ by definition of the stopping rule, hence
\[q_{1}^{t} = \sum_{a} p_1^{t+1}(s_1,a)\bar{E}_h^{t}(s_1,a) \geq (\varepsilon/2) \sum_{a \in \cA}\ind\left(\pi_1^{t+1}(s_1) = a\right) = \varepsilon/2.\]
Using \eqref{crucial} to upper bound $q_1^{t}$ yields
${\small
\varepsilon/2
 \leq  3 \sum_{h=1}^{H}\sum_{(s,a)} \gamma^{h-1} \sigma_{H-h}p_{h}^{t+1}(s,a)\left[\sqrt{\frac{\beta(n_h^t(s,a),\delta)}{n_h^t(s,a)}}\wedge 1\right]
}$ for $t < \tau$
and summing these inequalities for $t \in \{0,\dots,T\}$ where $T < \tau$ gives
\begin{align*}
(T+1) \varepsilon &\leq 6 \sum_{h=1}^{H} \gamma^{h-1}\sigma_{H-h}\sum_{(s,a)}\sum_{t=0}^{T} p_{h}^{t+1}(s,a)\left[\sqrt{\frac{\beta(n_h^t(s,a),\delta)}{n_h^t(s,a)}} \wedge 1 \right].
\end{align*}
The next step is to relate the counts to the pseudo-counts using the fact that the event $\cE^{\text{cnt}}$ holds. 

\begin{lemma}\label{lem:cnt_pseudo} On the event $\cE^{\text{cnt}}$, $\forall  h \in [H], (s,a) \in \cS \times \cA$,
\[ \forall t \in \N^*, \ \frac{\beta(n_h^t(s,a),\delta)}{n_h^t(s,a)}\wedge 1 \leq 4 \frac{\beta(\bar n_h^t(s,a),\delta)}{\bar n_h^t(s,a)\vee 1}\,.\]
\end{lemma}

Using Lemma~\ref{lem:cnt_pseudo}, one can write that, on the event $\cF$, for $T < \tau$,
\begin{align*}
 (T+1) \varepsilon &  \leq 12 \sum_{h=1}^{H} \gamma^{h-1}\sigma_{H-h}\sum_{(s,a)}\sum_{t=0}^{T} p_{h}^{t+1}(s,a)\sqrt{\frac{\beta(\bar n_h^t(s,a),\delta)}{\bar n_h^t(s,a)\vee 1}} \\
 &  \leq 12 \sqrt{\beta(T+1,\delta)}\sum_{h=1}^{H} \gamma^{h-1}\sigma_{H-h}\sum_{(s,a)}\sum_{t=0}^{T} \frac{\bar n_h^{t+1}(s,a) - \bar n_h^t(s,a)}{\sqrt{\bar n_h^t(s,a)\vee 1}},
\end{align*}
where we have used that by definition of the pseudo-counts $p_{h}^{t+1}(s,a) = \bar n_{h}^{t+1}(s,a) - \bar n_{h}^{t}(s,a)$. Using Lemma 19 of \cite{UCRL10} (recalled in Appendix~\ref{app:technical}) to upper bound the sum in $t$ yields
\begin{align*}
 (T+1) \varepsilon &  \leq 12(1+\sqrt{2}) \sqrt{\beta(T+1,\delta)}\sum_{h=1}^{H} \gamma^{h-1}\sigma_{H-h}\sum_{(s,a)}\sqrt{n_h^{T+1}(s,a)} \\
 &\leq 12(1+\sqrt{2}) \sqrt{\beta(T+1,\delta)}\sum_{h=1}^{H} \gamma^{h-1}\sigma_{H-h}\sqrt{SA}\sqrt{\sum_{s,a} n_h^{T+1}(s,a)}.
\end{align*}
As $\sum_{s,a} n_h^{T+1}(s,a) = T+1$, one obtains, using further that $\sigma_{H - h} \leq \sigma_H$,
\begin{align*}
\varepsilon \sqrt{T+1}  &  \leq 12(1+\sqrt{2})\sqrt{SA}\left(\sigma_H\sum_{h=1}^{H} \gamma^{h-1}\right)  \sqrt{\beta(T+1,\delta)}.
\end{align*}
For $T$ large enough, this inequality cannot hold, as the left hand side is in $\sqrt{T}$ while the right hand-side is logarithmic. Hence $\tau$ is finite and satisfies (applying the inequality to $T=\tau-1$)
\begin{align*}
\tau    \leq \frac{\cC_{H}SA}{\varepsilon^2}\beta(\tau,\delta),
\end{align*}
where $\cC_{H} = 144(1+\sqrt{2})^2\sigma_{H}^4$. The conclusion follows from Lemma~\ref{lemma:inversion_of_n_inequality} stated in Appendix~\ref{app:technical}.


\section{Reward-Free Exploration versus Best Policy Identification}
\label{sec:RF_to_BPI}

While originally proposed for solving the Best Policy Identification problem (see Definition~\ref{def:BPI}), we proved that \OurAlgorithm{} is $(\varepsilon,\delta)$-PAC for Reward-Free Exploration. In particular, \OurAlgorithm{} is also $(\varepsilon,\delta)$-PAC for BPI given some deterministic reward function $r$. In this section, we investigate the difference in complexity between BPI and RFE, trying to answer the following question: could an algorithm specifically designed for BPI have a smaller sample complexity than \OurAlgorithm{}?

A lower bound for BPI can be found in the work of \citet{dann15PAC} (although this work in focused on the design of PAC-MDP algorithms). This worse-case lower bound says that for any $(\epsilon,\delta)$-PAC algorithm for BPI there exists an MDP with stationary transitions for which $\bE[\tau] = \Omega\left(((SAH^2)/\varepsilon^2)\ln\left({c_2}/(\delta+c_3)\right)\right)$ for some constant $c_2$ and $c_3$. This lower bound directly translates to a lower bound for RFE, showing that for a small $\delta$ the sample complexity of \OurAlgorithm{} is optimal up to a factor $H$ for stationary rewards, for \emph{both} BPI and RFE. If one is interested in the small $\varepsilon$ regime, the lower bound of \citet{Jin20RewardFree} is more informative: it states that for any $(\epsilon,\delta)$-PAC algorithm for RFE there exists an MDP with stationary transitions for which $\bE[\tau] =  \Omega\left(((S^2AH^2)/\varepsilon^2)\right)$. Observe the increased $S^2$ factor, which may not be needed for a BPI algorithm to be optimal in the small $\varepsilon$ regime, and justifies the need to derive specific BPI algorithms.

\paragraph{BPI algorithms} To the best of our knowledge, the BPI problem has not been studied a lot since the work of \citet{Fiechter94}.  \citet{EvenDaral06} propose $(\epsilon,\delta)$-PAC algorithms that stop and output a guess for the optimal policy (in all states) for discounted MDPs, but no upper bound on their sample complexity is given. Another avenue to get an $(\varepsilon,\delta)$-PAC algorithm for BPI is to use a regret minimization algorithm: \citet{Jin18OptQL} suggests to run a regret minimization algorithm for some well chosen number of episode $K$ and to let $\hat{\pi}$ be a policy chosen at random among the $K$ policies used. Taking as a sub-routine the UCB-VI algorithm of \citet{Azar17UCBVI} that has $\cO(\sqrt{H^2SAK})$ regret for stationary rewards, with $K = \cO({H^2SA}/(\varepsilon^2\delta^2))$ this conversion yields an $(\epsilon,\delta)$-PAC for BPI. Its sample complexity has a bad scaling\footnote{{As pointed out to us, the scaling in $\delta$ can be improved to $\log^2(1/\delta)$ for a different static conversion from regret to BPI, see Appendix~\ref{app:ChiJin}}.} in $\delta$, but is optimal for BPI when $\varepsilon$ is small.

To get a better dependency in $\delta$, a first observation is that \OurAlgorithm{} can be used: being $(\varepsilon,\delta)$-PAC for RFE, it will also be $(\varepsilon,\delta)$-PAC for BPI (with a recommendation rule $\hat{\pi}_{\tau}=\hat{\pi}^\star_{\tau,r}$). Yet, the sampling rule of \OurAlgorithm{} does not leverage the knowledge of $r$, and intuitively, there is something to be gained by doing it. This is why we propose the \OurAlgorithmBPI{} algorithm, which does exploit the observation of the rewards during learning, and can be seen as an \emph{adaptive} conversion from regret to BPI. The algorithm, described in full details in Appendix~\ref{app:BPI}, equips a regret minimizing algorithm similar to the KL-UCRL algorithm of \citet{filippi2010optimism} with an adaptive stopping rule, which leverages upper and lower confidence bounds on the value functions.

\paragraph{BPI-UCRL}  Letting $\uQ_h^{t}(s,a)$ and $\lQ_h^{t}(s,a)$ be upper and lower confidence bounds on $Q^\star_h(s,a)$ that are defined in Appendix~\ref{app:BPI}, the three components of \OurAlgorithmBPI are
\begin{itemize}
  \item the \textbf{sampling rule} $\pi^{t+1}$, which is the greedy policy \emph{w.r.t.} to the upper bounds $\uQ_h^{t}(s,a)$,
  \item the \textbf{stopping rule} $\tau = \inf\big\{ t\in\N : \max_a \uQ_1^t(s_1,a)-\max_{a}\lQ_1^t(s_1,a) \leq \epsilon \big\}$ and
  \item the \textbf{recommendation rule} $\hpi_\tau$, which is the greedy policy \emph{w.r.t.} to lower bounds $\lQ_h^{\tau}(s,a)$\;.
\end{itemize}\color{black}
We prove in Theorem~\ref{thm:sc_BPI} that \OurAlgorithmBPI{} enjoys the same sample complexity guarantees than \OurAlgorithm{}, both in the small $\delta$ and the small $\varepsilon$ regime. 
Yet, as illustrated in the next section, \OurAlgorithmBPI{} appears to perform more efficient exploration than \OurAlgorithm{} for a fixed reward function. We leave as an open question whether an improved analysis for \OurAlgorithmBPI{} could corroborate this improvement (besides the slightly smaller constant $\cC_H$ in Theorem~\ref{thm:sc_BPI}).

\paragraph{Open questions} We summarize in Figure~\ref{table:12} below the best available upper and lower bounds available for RFE and BPI, when the transitions are stationary. While in RFE the \OurAlgorithm{} algorithm is optimal up to a factor $H$ in both the small $\delta$ and small $\varepsilon$ regimes, it is not clear whether an algorithm having this property exists for BPI. Figure~\ref{table:12} also shows that in the small $\varepsilon$ regime, the complexity of RFE and BPI are different as there exists an algorithm with sample complexity $SAH^2/\varepsilon^2$ for BPI while all RFE algorithms have a sample complexity that is larger than $S^2AH^2/\varepsilon^2$. In the small $\delta$ regime, designing algorithms whose sample complexity scales in $H^2$ instead of $H^3$ would allow to conclude that the complexity of the two problems in the same, at least in a worst-case sense. We leave this task as future work.


\setlength{\extrarowheight}{5pt}
\begin{figure}[h]
\begin{minipage}{0.5\linewidth}\hspace{-0.6cm}
\begin{tabular}{|c | c| c|}
 \hline
 Small $\delta$ & UB & LB  \\
 \hline\hline
 RFE &  $\frac{S A H^3}{\epsilon^2}\log\left(\frac{1}{\delta}\right)$  & $\frac{S A H^2}{\epsilon^2} \log\left(\frac{1}{\delta}\right)$ \\
 & \OurAlgorithm{} & \small \cite{dann15PAC} \hspace{-0.3cm}\\
 \hline
 BPI & $\frac{S A H^3}{\epsilon^2}\log\left(\frac{1}{\delta}\right)$ &  $\frac{S A H^2}{\epsilon^2} \log\left(\frac{1}{\delta}\right)$   \\
  & \OurAlgorithmBPI{} & \small \cite{dann15PAC} \hspace{-0.3cm} \\
  & \OurAlgorithm{} & \\
 \hline
\end{tabular}
\end{minipage}
\begin{minipage}{0.49\linewidth}\hspace{0.3cm}
\begin{tabular}{|c | c| c|}
 \hline
 Small $\epsilon$ & UB & LB  \\
 \hline\hline
 RFE &  $\frac{S^2 A H^3}{\epsilon^2}$ & $\frac{S^2 A H^2}{\epsilon^2}$ \\
  & \OurAlgorithm{} & \small \cite{Jin20RewardFree} \hspace{-0.3cm} \\
 \hline
 BPI & $\frac{S A H^2}{\epsilon^2}$ &  $\frac{S A H^2}{\epsilon^2}$   \\
  & UCB-VI & \small \cite{dann15PAC} \hspace{-0.3cm}\\
  &  & \\
 \hline
\end{tabular}
\end{minipage}
\caption{Available upper and lower bounds on the sample complexity for RFE and BPI for stationary transition kernels ($p_h(s' | s,a) = p(s' |s,a)$) in a regime of small $\delta$ (left) and small $\varepsilon$ (right)\label{table:12}}
\end{figure}

\color{black}

\section{Numerical Illustration} \label{sec:experiments}

In this section we report the results of some experiments on synthetic MDPs aimed at illustrating how \OurAlgorithm{} and \OurAlgorithmBPI{} perform exploration, compared to simple baselines:
(i) exploration with a random policy (RP) agent, and
(ii) a generative model (GM) agent, which samples a fixed number of transitions from each state-action pair.

\begin{figure}[ht]
	\floatconts{fig:results}
	{\caption{\OurAlgorithm{} and \OurAlgorithmBPI{} on the DoubleChain MDP, with parameters $L = 31$, $H=20$.}}
	{
		\subfigure[Approximation error as a function of $n$]{
			\includegraphics[trim={0 0 1cm 1cm}, clip, width=0.45\linewidth]{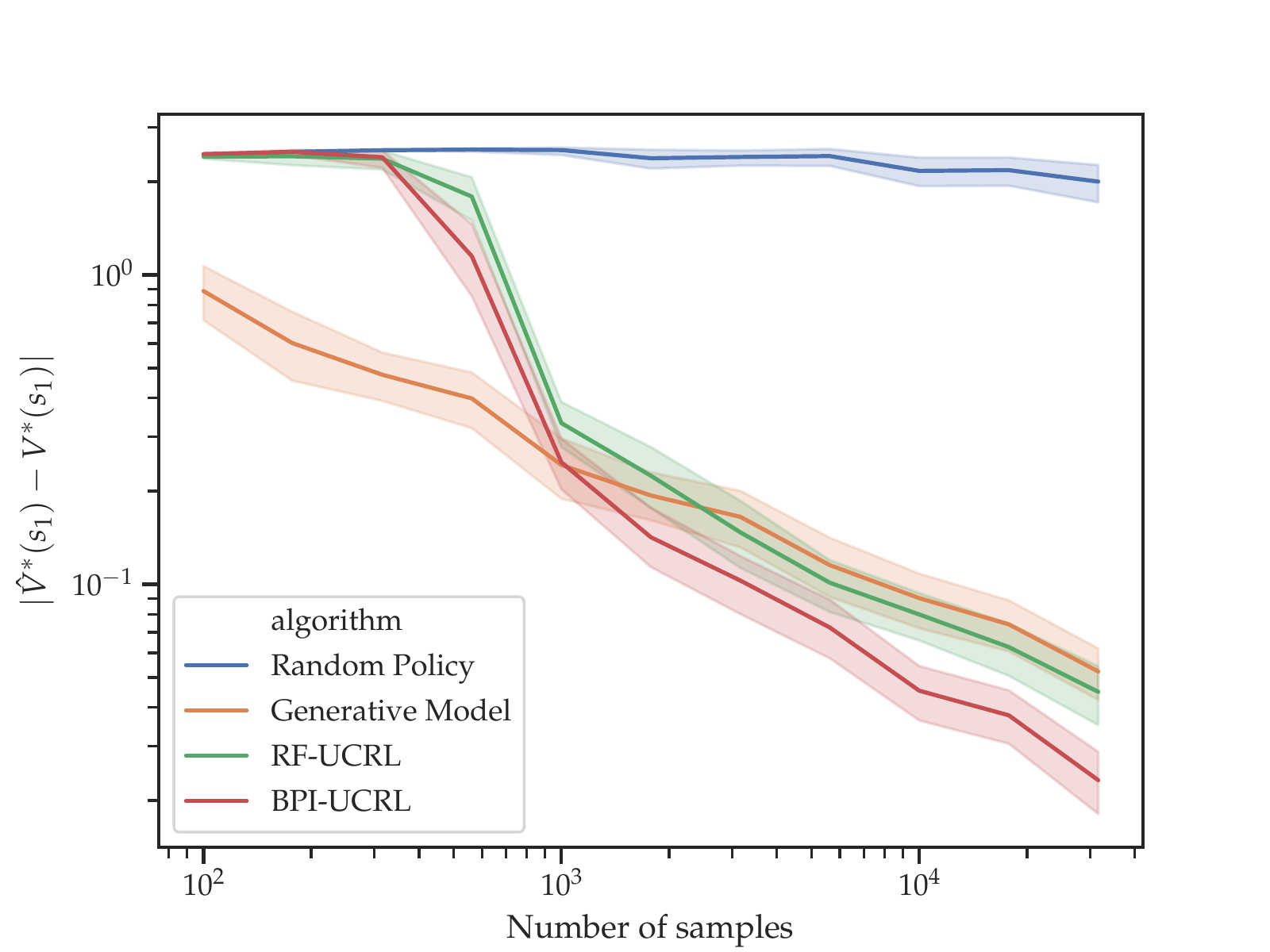}
			\label{fig:chain-error}
		}\quad
		\subfigure[Number of state visits for $n = 5000$]{
			\includegraphics[trim={0 0 1cm 1cm}, clip, width=0.45\linewidth]{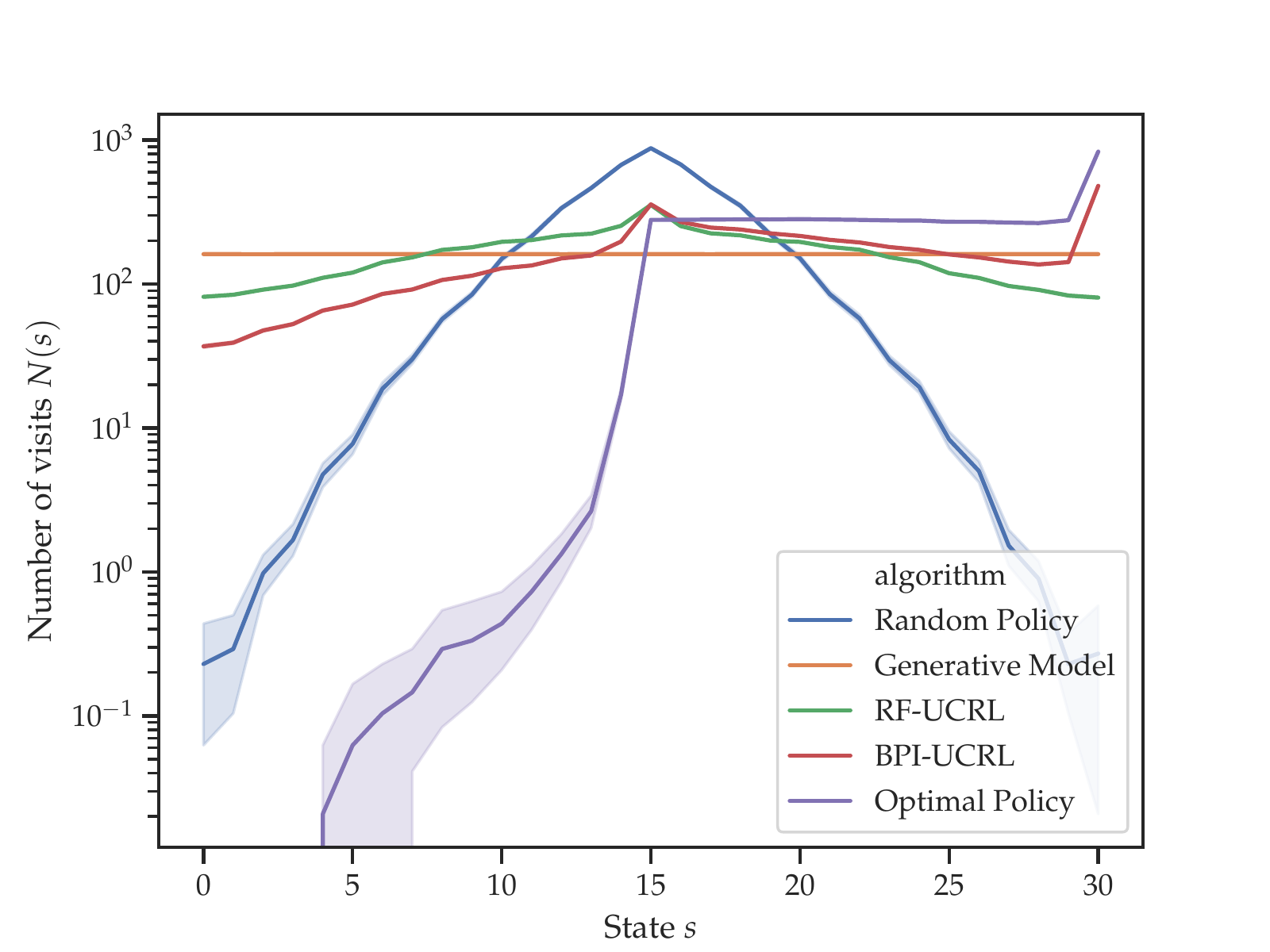}
			\label{fig:chain-occupancies}
		}
		\subfigure[{Estimation of $\bE[\tau | \tau \!< 10^8]$ for \OurAlgorithm{}}]{
			\includegraphics[trim={0 0 1cm 1cm}, clip, width=0.45\linewidth]{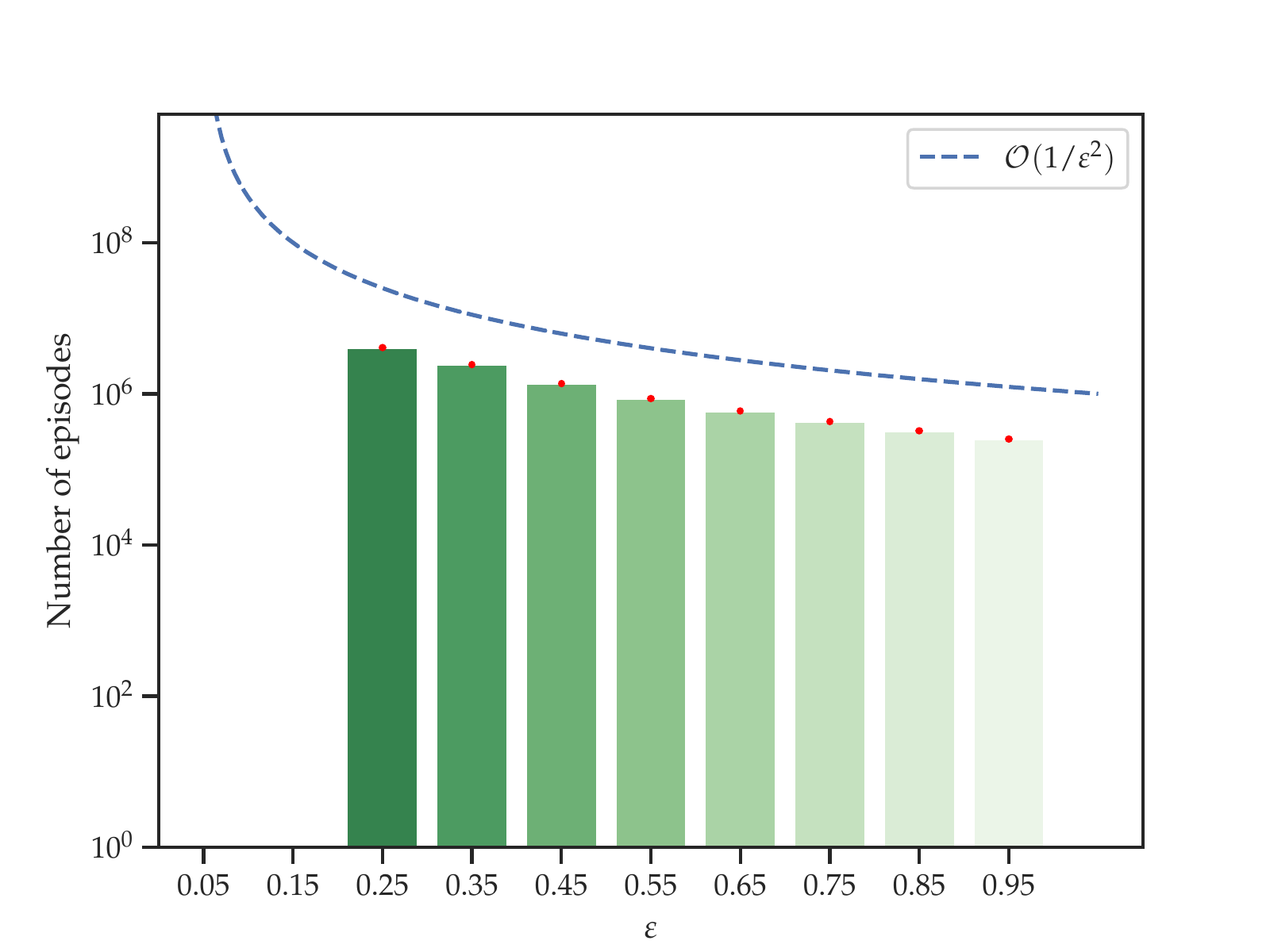}
			\label{fig:error-bins-RFE}
		}\quad
		\subfigure[{Estimation of $\bE[\tau | \tau \! < 10^6]$ for \OurAlgorithmBPI{}}]{
			\includegraphics[trim={0 0 1cm 1cm}, clip, width=0.45\linewidth]{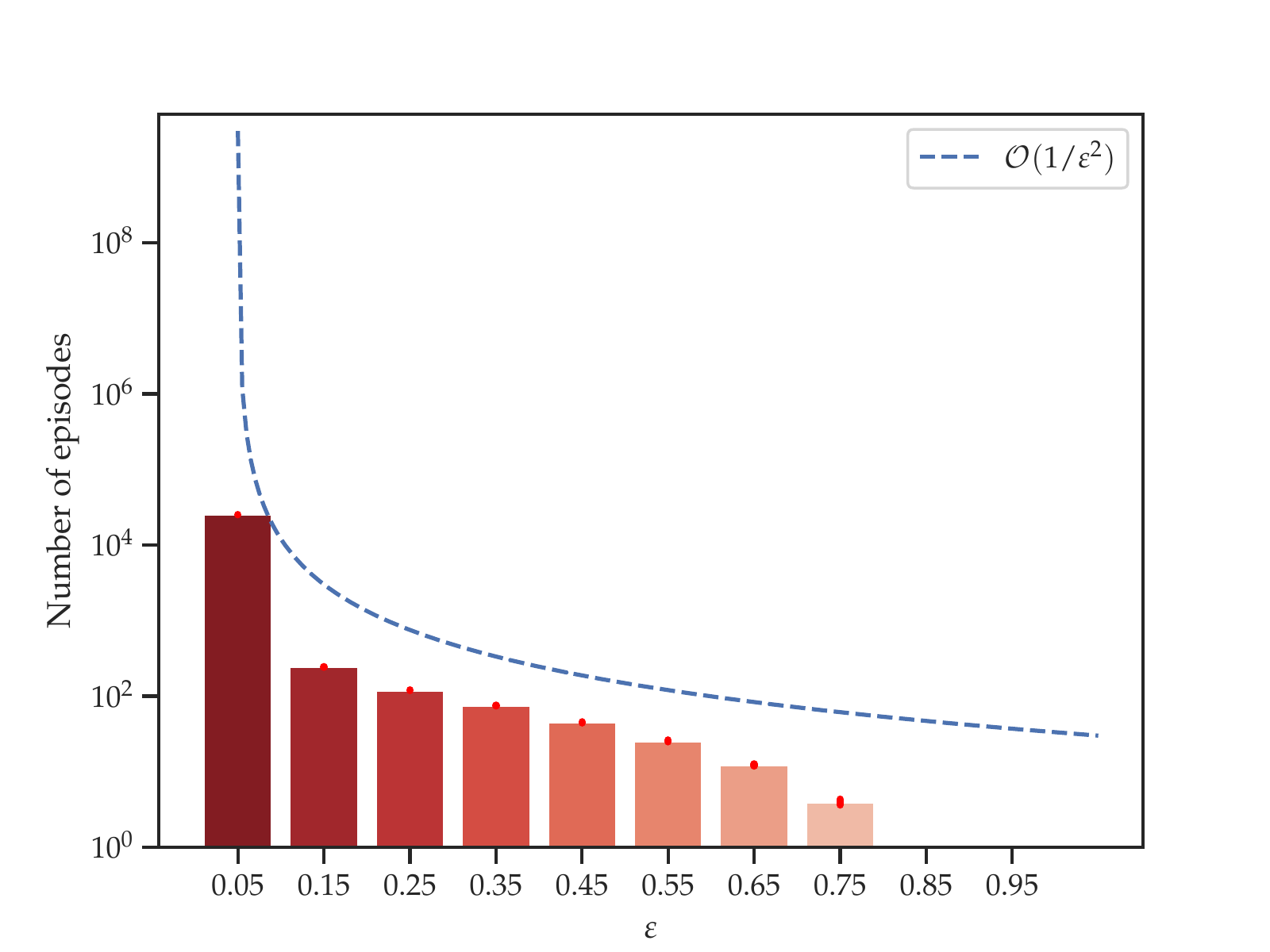}
			\label{fig:error-bins-BPI}
		}
	}

\end{figure}

We perform the following experiment: each algorithm interacts with the environment until it gathers a total of $n$ transitions (exploration phase). RP, GM and \OurAlgorithm{} use the gathered data to estimate a model $\hat{P}$ and, at the end, they are given a reward function $r$ and compute $\widehat{V}^\star_1(s_1; r)$ (estimation phase). Only the \OurAlgorithmBPI{} agent is allowed to observe the rewards during exploration phase, after which it outputs a policy $\widehat{\pi}^\star$. For \OurAlgorithm{} and \OurAlgorithmBPI{} we use the threshold $\beta(n,\delta)$ specified in Theorems~\ref{thm:sc} and \ref{thm:sc_BPI} with $\delta = 0.1$. For \OurAlgorithm{} we found out that removing the minimum with $\gamma \sigma_{H-h}$ in the definition of the error bound \eqref{def:bound} (which still gives a valid high-probability upper bound) leads to better practical performance, and we report results for this variant. Our study does not include the parallel regret minimization approach of \cite{Jin20RewardFree} described in the Introduction
as this algorithm is mostly theoretical: the parameters $N_0$ and $N$ that ensure the $(\varepsilon,\delta)$-PAC property are only given up to non-specified multiplicative constants.

We consider a Double Chain MDP, with states $\cS = \cpa{0,\ldots, L-1}$, where $L$ is the length of the chain, and actions $\cA = \cpa{0, 1}$, which correspond to a transition to the left (action $0$) or to the right (action $1$). When taking an action, there is a $0.1$ probability of moving to the other direction. A single reward of $1$ is placed at the rightmost state $s = L-1$, and the agent starts at $s_1 = (L-1)/2$, which leaves two possible directions for exploration. \Cref{fig:chain-error} shows the estimation error $|\widehat{V}^\star_1(s_1; r)-V^\star_1(s_1; r)|$ as a function of $n$, estimated over $N = 48$ runs. As expected, this error decays the fastest under \OurAlgorithmBPI{}, since its exploration is guided by the observed rewards. 
Interestingly, the performance of \OurAlgorithm{} is close to the agent which has access to a generative model.
\Cref{fig:chain-occupancies} shows the number of visits to each state during the exploration phase. We observe that the random policy is not able to reach the borders of the chain and, by design, the GM agent uniformly distributes its number of visits. \OurAlgorithm{} actively seeks to sample less visited states, and manage to fully explore the chain, whereas BPI-UCRL focuses its exploration on the part of the chain where the highest reward is placed.
In Appendix \ref{app:numerical_illustration} we report additional results of experiments in a GridWorld, where we observe the same behavior. In Figure~\ref{fig:error-bins-RFE} and \ref{fig:error-bins-BPI} we estimate the sample complexity $\tau$ of \OurAlgorithm{} and \OurAlgorithmBPI{} for different values of $\varepsilon$ using $N=48$ runs of $n=10^8$ (resp. $n=10^6$) sampled transitions, and checking the time needed for stopping at a level $\varepsilon$ (if stopping occurs before $n$). \OurAlgorithmBPI{} has indeed a much smaller sample complexity. 

\section{Conclusion}

Inspired by the work of Fiechter from 1994, we proposed Reward-Free UCRL, a natural adaptive approach to Reward-Free Exploration. The improved sample complexity of this $(\varepsilon,\delta)$-PAC algorithm is matching the existing lower bounds up to a factor $H$ in both regimes of small $\varepsilon$ and small $\delta$. We also proposed BPI-UCRL for the related Best Policy Identification problem that was the initial focus of the work of Fiechter. Understanding the difference in complexity between RFE and BPI is an interesting open question, not fully solved in this paper. In particular, we will investigate in future work whether it is possible to design algorithms that are simultaneously optimal in the small $\delta$ and small $\varepsilon$ regime for RFE or BPI. For BPI, we believe that one should look beyond the existing worse-case lower bound for problem dependent guarantees.

\section*{Acknowledgements} We acknowledge the support of the European CHISTERA project DELTA. Anders Jonsson is partially supported by the Spanish grants TIN2015-67959 and PCIN-2017-082.

\bibliography{biblio_alt}

\begin{thebibliography}{32}
\providecommand{\natexlab}[1]{#1}
\providecommand{\url}[1]{\texttt{#1}}
\expandafter\ifx\csname urlstyle\endcsname\relax
  \providecommand{\doi}[1]{doi: #1}\else
  \providecommand{\doi}{doi: \begingroup \urlstyle{rm}\Url}\fi

\bibitem[Azar et~al.(2012)Azar, Munos, and Kappen]{Azar12SCGene}
Mohammad~Gheshlaghi Azar, R{\'{e}}mi Munos, and Bert Kappen.
\newblock On the sample complexity of reinforcement learning with a generative
  model.
\newblock In \emph{Proceedings of the 29th International Conference on Machine
  Learning (ICML)}, 2012.

\bibitem[Azar et~al.(2017)Azar, Osband, and Munos]{Azar17UCBVI}
Mohammad~Gheshlaghi Azar, Ian Osband, and R{\'{e}}mi Munos.
\newblock Minimax regret bounds for reinforcement learning.
\newblock In \emph{Proceedings of the 34th International Conference on Machine
  Learning, (ICML) 2017}, 2017.

\bibitem[Brafman and Tennenholtz(2002)]{Brafman02RMAX}
Ronen~I. Brafman and Moshe Tennenholtz.
\newblock {R-MAX} - {A} general polynomial time algorithm for near-optimal
  reinforcement learning.
\newblock \emph{Journal of Machine Learning Research}, 3:\penalty0 213--231,
  2002.

\bibitem[Chentanez et~al.(2005)Chentanez, Barto, and Singh]{NIPS2004_2552}
Nuttapong Chentanez, Andrew~G. Barto, and Satinder~P. Singh.
\newblock Intrinsically motivated reinforcement learning.
\newblock In L.~K. Saul, Y.~Weiss, and L.~Bottou, editors, \emph{Advances in
  Neural Information Processing Systems 17}, pages 1281--1288. MIT Press, 2005.
\newblock URL
  \url{http://papers.nips.cc/paper/2552-intrinsically-motivated-reinforcement-learning.pdf}.

\bibitem[Cohen et~al.(2020)Cohen, Kaplan, Mansour, and
  Rosenberg]{cohen2020near}
Alon Cohen, Haim Kaplan, Yishay Mansour, and Aviv Rosenberg.
\newblock Near-optimal regret bounds for stochastic shortest path.
\newblock \emph{arXiv preprint arXiv:2002.09869}, 2020.

\bibitem[Dann and Brunskill(2015)]{dann15PAC}
Christoph Dann and Emma Brunskill.
\newblock Sample complexity of episodic fixed-horizon reinforcement learning.
\newblock In \emph{Advances in Neural Information Processing Systems (NIPS)},
  2015.

\bibitem[Dann et~al.(2017)Dann, Lattimore, and Brunskill]{dann2017unifying}
Christoph Dann, Tor Lattimore, and Emma Brunskill.
\newblock Unifying pac and regret: Uniform pac bounds for episodic
  reinforcement learning.
\newblock In \emph{Advances in Neural Information Processing Systems}, pages
  5713--5723, 2017.

\bibitem[Even-Dar et~al.(2006)Even-Dar, Mannor, and Mansour]{EvenDaral06}
E.~Even-Dar, S.~Mannor, and Y.~Mansour.
\newblock {Action Elimination and Stopping Conditions for the Multi-Armed
  Bandit and Reinforcement Learning Problems}.
\newblock \emph{Journal of Machine Learning Research}, 7:\penalty0 1079--1105,
  2006.

\bibitem[Fiechter(1994)]{Fiechter94}
Claude{-}Nicolas Fiechter.
\newblock Efficient reinforcement learning.
\newblock In \emph{Proceedings of the Seventh Conference on Computational
  Learning Theory (COLT)}, 1994.

\bibitem[Filippi et~al.(2010)Filippi, Capp{\'e}, and
  Garivier]{filippi2010optimism}
S.~Filippi, O.~Capp{\'e}, and A.~Garivier.
\newblock {Optimism in Reinforcement Learning and {K}ullback-{L}eibler
  Divergence}.
\newblock In \emph{{Allerton Conference on Communication, Control, and
  Computing}}, 2010.

\bibitem[Gajane et~al.(2019)Gajane, Ortner, Auer, and
  Szepesvari]{gajane2019autonomous}
Pratik Gajane, Ronald Ortner, Peter Auer, and Csaba Szepesvari.
\newblock Autonomous exploration for navigating in non-stationary cmps, 2019.

\bibitem[Hazan et~al.(2018)Hazan, Kakade, Singh, and Soest]{hazan2018provably}
Elad Hazan, Sham~M. Kakade, Karan Singh, and Abby~Van Soest.
\newblock Provably efficient maximum entropy exploration.
\newblock \emph{arXiv preprint arXiv:1812.02690}, 2018.

\bibitem[Jaksch et~al.(2010)Jaksch, Ortner, and Auer]{UCRL10}
Thomas Jaksch, Ronald Ortner, and Peter Auer.
\newblock Near-optimal regret bounds for reinforcement learning.
\newblock \emph{Journal of Machine Learning Research}, 11:\penalty0 1563--1600,
  2010.

\bibitem[Jin et~al.(2018)Jin, Allen{-}Zhu, Bubeck, and Jordan]{Jin18OptQL}
Chi Jin, Zeyuan Allen{-}Zhu, S{\'{e}}bastien Bubeck, and Michael~I. Jordan.
\newblock Is q-learning provably efficient?
\newblock In \emph{Advances in Neural Information Processing Systems
  (NeurIPS)}, 2018.

\bibitem[Jin et~al.(2020)Jin, Krishnamurthy, Simchowitz, and
  Yu]{Jin20RewardFree}
Chi Jin, Akshay Krishnamurthy, Max Simchowitz, and Tiancheng Yu.
\newblock Reward-free exploration for reinforcement learning.
\newblock \emph{arXiv:2002.02794}, 2020.

\bibitem[Jonsson et~al.(2020)Jonsson, Kaufmann, M\'enard, Darwiche-Domingues,
  Leurent, and Valko]{MDPGAPE}
Anders Jonsson, Emilie Kaufmann, Pierre M\'enard, Omar Darwiche-Domingues,
  Edouard Leurent, and Michal Valko.
\newblock Planning in markov decision processes with gap-dependent sample
  complexity.
\newblock In \emph{Advances in Neural Processing Systems (NeurIPS)}, 2020.

\bibitem[Kakade(2003)]{Kakade03PhD}
Sham Kakade.
\newblock \emph{On the Sample Complexity of Reinforcement Learning}.
\newblock PhD thesis, University College London, 2003.

\bibitem[Kearns and Singh(1998)]{KearnsS98MBQVI}
Michael~J. Kearns and Satinder~P. Singh.
\newblock Finite-sample convergence rates for q-learning and indirect
  algorithms.
\newblock In \emph{Advances in Neural Information Processing Systems (NIPS)},
  pages 996--1002, 1998.

\bibitem[Kearns and Singh(2002)]{kearns02E3}
Michael~J. Kearns and Satinder~P. Singh.
\newblock Near-optimal reinforcement learning in polynomial time.
\newblock \emph{Machine Learning}, 49\penalty0 (2-3):\penalty0 209--232, 2002.

\bibitem[Lim and Auer(2012)]{lim2012autonomous}
Shiau~Hong Lim and Peter Auer.
\newblock Autonomous exploration for navigating in mdps.
\newblock In \emph{Conference on Learning Theory}, pages 40--1, 2012.

\bibitem[Mohamed and Jimenez~Rezende(2015)]{NIPS2015_5668}
Shakir Mohamed and Danilo Jimenez~Rezende.
\newblock Variational information maximisation for intrinsically motivated
  reinforcement learning.
\newblock In C.~Cortes, N.~D. Lawrence, D.~D. Lee, M.~Sugiyama, and R.~Garnett,
  editors, \emph{Advances in Neural Information Processing Systems 28}, pages
  2125--2133. Curran Associates, Inc., 2015.
\newblock URL
  \url{http://papers.nips.cc/paper/5668-variational-information-maximisation-for-intrinsically-motivated-reinforcement-learning.pdf}.

\bibitem[Montufar et~al.(2016)Montufar, Ghazi-Zahedi, and
  Ay]{montufar2016information}
Guido Montufar, Keyan Ghazi-Zahedi, and Nihat Ay.
\newblock Information theoretically aided reinforcement learning for embodied
  agents, 2016.

\bibitem[Ostrovski et~al.(2017)Ostrovski, Bellemare, van~den Oord, and
  Munos]{ostrovski2017count}
Georg Ostrovski, Marc~G Bellemare, A{\"a}ron van~den Oord, and R{\'e}mi Munos.
\newblock Count-based exploration with neural density models.
\newblock In \emph{Proceedings of the 34th International Conference on Machine
  Learning-Volume 70}, pages 2721--2730. JMLR. org, 2017.

\bibitem[Schmidhuber(1991)]{schmidhuber1991possibility}
J{\"u}rgen Schmidhuber.
\newblock A possibility for implementing curiosity and boredom in
  model-building neural controllers.
\newblock In \emph{Proc. of the international conference on simulation of
  adaptive behavior: From animals to animats}, pages 222--227, 1991.

\bibitem[Still and Precup(2012)]{PMID:22791268}
Susanne Still and Doina Precup.
\newblock An information-theoretic approach to curiosity-driven reinforcement
  learning.
\newblock \emph{Theory in biosciences = Theorie in den Biowissenschaften},
  131\penalty0 (3):\penalty0 139—148, September 2012.
\newblock ISSN 1431-7613.
\newblock \doi{10.1007/s12064-011-0142-z}.
\newblock URL \url{https://doi.org/10.1007/s12064-011-0142-z}.

\bibitem[Strehl and Littman(2008)]{Strehl08MBIE}
Alexander~L. Strehl and Michael~L. Littman.
\newblock An analysis of model-based interval estimation for markov decision
  processes.
\newblock \emph{Journal of Computer and System Sciences}, 74\penalty0
  (8):\penalty0 1309--1331, 2008.

\bibitem[Strehl et~al.(2006)Strehl, Li, Wiewiora, Langford, and
  Littman]{Strehl06DelayedQL}
Alexander~L. Strehl, Lihong Li, Eric Wiewiora, John Langford, and Michael~L.
  Littman.
\newblock {PAC} model-free reinforcement learning.
\newblock In \emph{Proceedings of the Twenty-Third International Conference on
  Machine Learning (ICML}, 2006.

\bibitem[Tang et~al.(2017)Tang, Houthooft, Foote, Stooke, Chen, Duan, Schulman,
  DeTurck, and Abbeel]{tang2017exploration}
Haoran Tang, Rein Houthooft, Davis Foote, Adam Stooke, Xi~Chen, Yan Duan, John
  Schulman, Filip DeTurck, and Pieter Abbeel.
\newblock \# exploration: A study of count-based exploration for deep
  reinforcement learning.
\newblock In \emph{Advances in neural information processing systems}, pages
  2753--2762, 2017.

\bibitem[Tarbouriech et~al.(2019)Tarbouriech, Garcelon, Valko, Pirotta, and
  Lazaric]{tarbouriech2019no}
Jean Tarbouriech, Evrard Garcelon, Michal Valko, Matteo Pirotta, and Alessandro
  Lazaric.
\newblock No-regret exploration in goal-oriented reinforcement learning.
\newblock \emph{arXiv preprint arXiv:1912.03517}, 2019.

\bibitem[Wang et~al.(2020)Wang, Du, Yang, and Salakhutdinov]{wang2020on}
Ruosong Wang, Simon~S Du, Lin~F Yang, and Ruslan Salakhutdinov.
\newblock {On reward-free reinforcement learning with linear function
  approximation}.
\newblock \emph{arXiv preprint arXiv:2006.11274}, 2020.
\newblock URL \url{https://arxiv.org/pdf/2006.11274.pdf}.

\bibitem[Zanette and Brunskill(2019)]{Zanette19Euler}
Andrea Zanette and Emma Brunskill.
\newblock Tighter problem-dependent regret bounds in reinforcement learning
  without domain knowledge using value function bounds.
\newblock In \emph{Proceedings of the 36th International Conference on Machine
  Learning, (ICML)}, 2019.

\bibitem[Zhang et~al.(2020)Zhang, Ma, and Singla]{zhang2020task-agnostic}
Xuezhou Zhang, Yuzhe Ma, and Adish Singla.
\newblock {Task-agnostic exploration in reinforcement learning}.
\newblock \emph{arXiv preprint: arXiv:2006.09497}, 2020.
\newblock URL \url{https://arxiv.org/pdf/2006.09497.pdf}.

\end{thebibliography}

\appendix

\newpage

\section{Additional Experiments}
\label{app:numerical_illustration}

Here, we consider a GridWorld environment, whose state space is a set of discrete points in a $21\times21$ grid. In each state, an agent can choose four actions: left, right, up or down, and it has a $5\%$ probability of moving to the wrong direction. The reward is equal to 1 in state $(16, 16)$ and is 0 elsewhere. \Cref{fig:grid-error} shows $\abs{\widehat{V}_1^\star(s_1; r)-V^\star_1(s_1; r)}$ as a function of $n$ and \Cref{fig:grid-occupancies} shows the number of visits to each state during the exploration phase. We observe the same behavior as explained for the DoubleChain: RF-UCRL seeks to sample from less visited states, whereas BPI-UCRL focuses its exploration near the rewarding state $(16, 16)$.

\begin{figure}[H]
	\floatconts{fig:suppl}
	{\caption{Comparison of several algorithms on the Gridworld MDP.}}
	{
		\subfigure[Approximation error as a function of $n$]{
			\includegraphics[trim={0 0 0 0}, clip, width=0.5\linewidth]{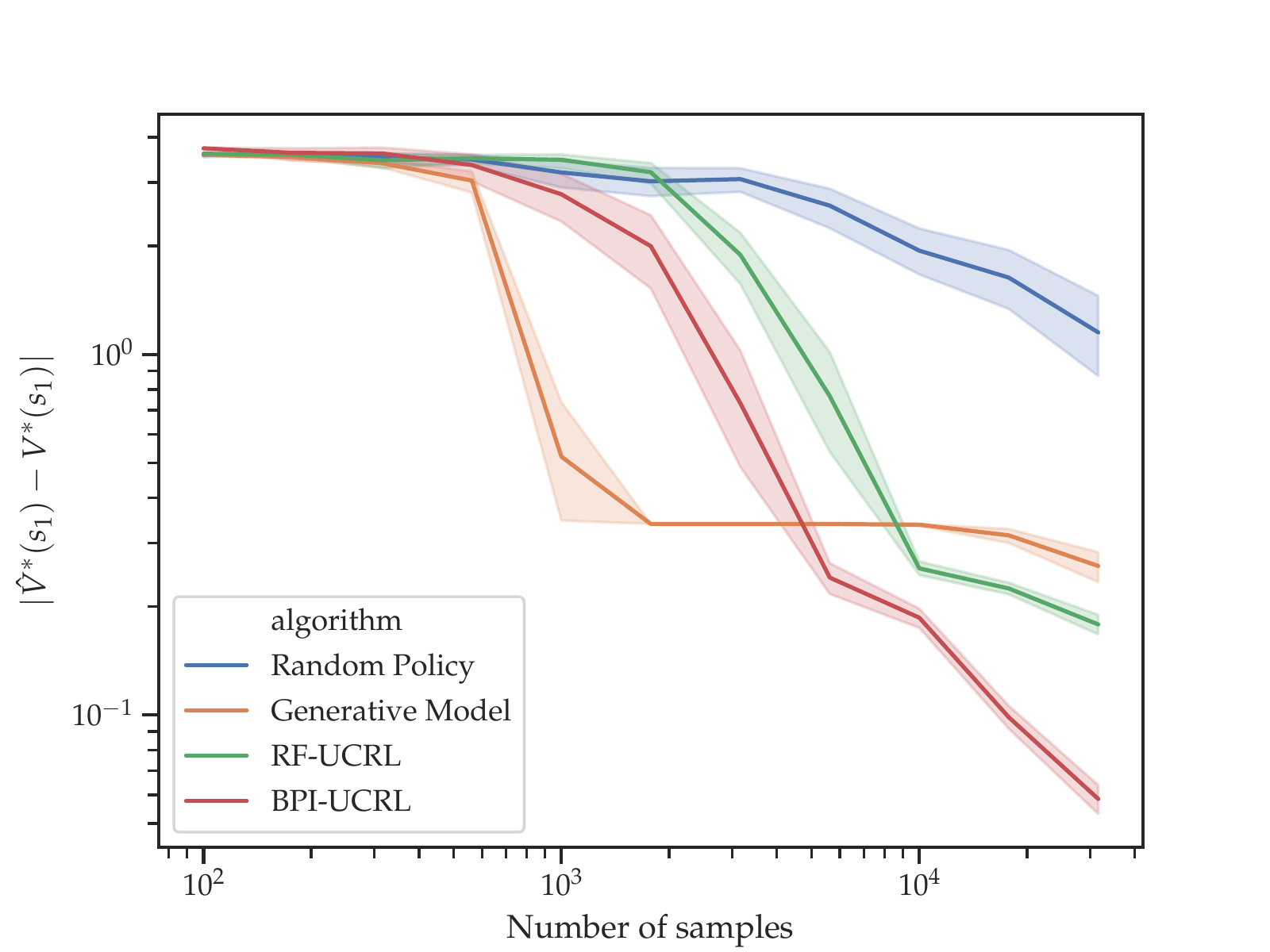}
			\label{fig:grid-error}
		}
		\subfigure[Number of state visits for $n = 30000$]{
			\includegraphics[trim={1cm 1cm 1cm 1cm}, clip, width=0.8\linewidth]{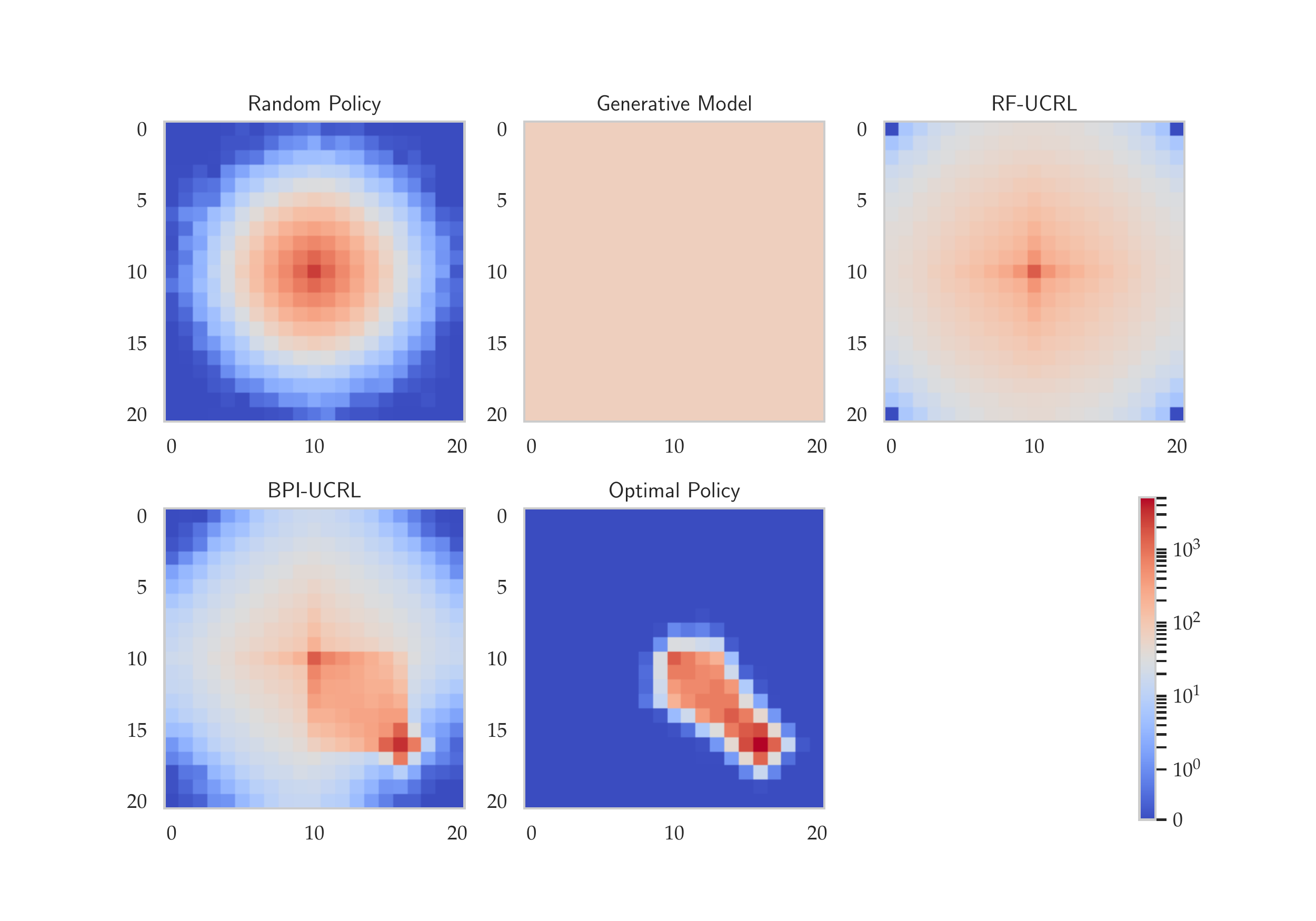}
			\label{fig:grid-occupancies}
		}
	}
\end{figure}

\section{High Probability Events}
\label{app:concentration_events}

We recall that the event $\cF$ introduced in the proof of Theorem~\ref{thm:sc} is the intersection of the two events 
\begin{eqnarray*}
  \cE & = & \left\{\forall t \in \N, \forall h \in [H], \forall (s,a), \KL\big(\hat{p}^t_h(\cdot | (s,a)), p_h(\cdot | (s,a))\big)\leq \frac{\beta(n_h^t(s,a),\delta)}{n_h^t(s,a)}\right\}\;, \\
  \cE^{\text{cnt}} & = & \left\{ \forall t \in \N^\star, \forall h\in [H],\forall (s ,a)\in\cS\times\cA:\ n_h^t(s,a) \geq \frac{1}{2}\bar n_h^t(s,a)-\log\left(\frac{2SAH}{\delta}\right)  \right\}\;.
\end{eqnarray*}

We first recall some useful concentration inequalities. The first one is a time-uniform deviation inequality for categorial random variable, proved by \cite{MDPGAPE}. 

\begin{lemma}[Proposition 1 in \cite{MDPGAPE}]\label{lem:max_ineq_categorical}Let $X_1, X_2,\ldots,X_n,\ldots$ be i.i.d. samples from a distribution supported over $\{1,\ldots,m\}$, of probabilities given by $p\in\Sigma_m$, where $\Sigma_m$ is the probability simplex of dimension $m-1$. We denote by $\hp_n$ the empirical vector of probabilities, i.e. for all $k\in\{1,\ldots,m\}$
\[
\hp_{n,k} = \frac{1}{n} \sum_{\ell=1}^n \ind(X_\ell = k)\,.
\]
For all $p\in\Sigma_m$, for all $\delta\in[0,1]$,
\[
    \PP\Big(\exists n\in \N^*,\, n\KL(\hp_n, p)> \log(1/\delta) + (m-1)\log\big(e(1+n/(m-1))\big)\Big)\leq \delta\,.
\]

\end{lemma}

The second is time-uniform deviation inequality for a sequence of Bernoulli random variables, proved by \citet{dann2017unifying}.

\begin{lemma}[Lemma F.4 in \cite{dann2017unifying}]
\label{lem:max_ineq_bernoulli_martingale} Let $X_1, X_2,\ldots,X_n,\ldots$ be a sequence of Bernoulli random variables adapted to the filtration $(\cF_t)_{t\in\N}$.
If we denote $p_n= \PP(X_n =1|\cF_{n-1})$, then for all $\delta\in(0,1]$
\[
\PP\left(\exists n \in \N^*: \sum_{\ell=1}^n X_\ell < \sum_{\ell=1}^n p_\ell/2 - \log(1/\delta)\right) \leq \delta \,.
\]
\end{lemma}

We can now prove the following.

\begin{lemma} \label{lem:concentration} For $\beta(n,\delta) = \log\!\big(2SAH/\delta\big) + (S-1)\log\big(e(1+n/(S-1))\big)$, it holds that $\bP\left(\cE\right) \geq 1- \tfrac{\delta}{2}$. Moreover, $\bP\left(\cE^{\text{cnt}}\right) \geq 1- \tfrac{\delta}{2}$.
\end{lemma}

\begin{proof} Using Lemma~\ref{lem:max_ineq_categorical} and a union bound yields
\begin{eqnarray*}
\bP\left(\cE^c\right) & \leq & \sum_{h = 1}^{H} \sum_{(s,a) \in \cS \times \cA} \bP\left(\exists t \in \N : {n_h^t(s,a)}\KL\big(\hat{p}^t_h(\cdot | (s,a)), p_h(\cdot | (s,a))\big)\geq \beta(n_h^t(s,a),\delta)\right) \\
& \leq & \sum_{h = 1}^{H} \sum_{(s,a) \in \cS \times \cA}  \frac{\delta}{2SAH} = \frac{\delta}{2}.
\end{eqnarray*} 

Then, using Lemma~\ref{lem:max_ineq_bernoulli_martingale} and a union bound yields
\begin{eqnarray*}
\bP\left(\left(\cE^{\text{cnt}}\right)^c\right) & \leq & \sum_{h = 1}^{H} \sum_{(s,a) \in \cS \times \cA} \bP\left(\exists t \in \N : n_h^t(s,a) \leq \frac{1}{2} \bar n_h^t(s,a) - \log\left(\frac{2SAH}{\delta}\right)\right) \\ 
& \leq & \sum_{h = 1}^{H} \sum_{(s,a) \in \cS \times \cA} \bP\left(\exists t \in \N : \sum_{i=1}^{t} \ind\left((s_h^i,a_h^i) = (s,a)\right) \leq \frac{1}{2} \sum_{i=1}^{t} p^{i}_h(s,a) - \log\left(\frac{2SAH}{\delta}\right)\right) \\ 
& \leq & \sum_{h = 1}^{H} \sum_{(s,a) \in \cS \times \cA}  \frac{\delta}{2SAH} = \frac{\delta}{2}.
\end{eqnarray*} 
\end{proof}

\section{Proof of Auxiliary Lemmas for Theorem~\ref{thm:sc}}\label{app:main_thm}

\subsection{From Values to Optimal Values}\label{proof:inclusion}

In this section, we prove the inclusion \eqref{estimation_to_value}, namely that for all $t$
\[\left\{ \forall \pi, \left|\hat V_1^{t,\pi}(s_1; r) - V_1^{\pi}(s_1; r) \right| \leq \varepsilon/2\right\} \subseteq \left\{ V_1^{\star}(s_1;r) - V_1^{\hat{\pi}^\star_{t,r}}(s_1;r) \leq \varepsilon\right\}.\]
We denote by $\pi^\star$ the optimal policy in the MDP $(P,r)$ and recall that $\hat{\pi}^\star_{t,r}$ is the optimal policy in the MDP $(\hat{P}_t,r)$. One can write 
\begin{eqnarray*}
 V_1^{\star}(s_1;r) - V_1^{\hat{\pi}^\star_{t,r}}(s_1;r) & = & V_1^{\pi^\star}(s_1;r) - \hat V_1^{t,\pi^\star}(s_1;r) + \underbrace{\hat V_1^{t,\pi^\star}(s_1;r) - \hat V_1^{t,\hat{\pi}^\star_{t,r}}(s_1;r)}_{\leq 0} \\ & & + \hat V_1^{t,\hat{\pi}^\star_{t,r}}(s_1;r) - V_1^{\hat{\pi}^\star_{t,r}}(s_1;r).
\end{eqnarray*}
The middle term is non-negative as $\hat{\pi}^\star_{t,r}$ is the optimal policy in the empirical MDP which yields 
\begin{eqnarray*}
 V_1^{\star}(s_1;r) - V_1^{\hat{\pi}^\star_{t,r}}(s_1;r) & \leq & \left|V_1^{\pi^\star}(s_1;r) - \hat V_1^{t,\pi^\star}(s_1;r)\right| + \left|\hat V_1^{t,\hat{\pi}^\star_{t,r}}(s_1;r) - V_1^{\hat{\pi}^\star_{t,r}}(s_1;r)\right|
\end{eqnarray*}
and easily yields the inclusion above. 

\subsection{Proof of Lemma~\ref{lem:induction}}

By definition of $\bar{E}_h^{t}(s,a)$ and the greedy policy $\pi^{t+1}$, if $n_h^t(s,a) > 0$,
\begin{eqnarray}\overline{E}_h^{t}(s,a) &\leq& \gamma \sigma_{H-h}\sqrt{\frac{2\beta(n_h^t(s,a),\delta)}{n_h^t(s,a)}} + \gamma \sum_{s' \in \cS} \hat p_h^{t}(s' | s,a) \bar E^t_{h+1}(s', {\pi}^{t+1}(s'))\label{ToPlugIn}\,.\end{eqnarray}

From the definition of the event $\cE$ and Pinsker's inequality, one can further upper bound 

{\footnotesize
\begin{align*}\sum_{s' \in \cS}\! \hat p_h^{t}(s' | s,a) \bar E^t_{h+1}(s', {\pi}^{t+1}(s'))& \leq \sum_{s' \in \cS} \! p_h^{t}(s' | s,a) \bar E^t_{h+1}(s', {\pi}^{t+1}(s')) + \sum_{s' \in \cS} \! \left(\hat p_h^{t}(s' | s,a) - p_h^{t}(s' | s,a)\right)\bar E^t_{h+1}(s', {\pi}^{t+1}(s')) \\
& \leq   \sum_{s' \in \cS}\!  p_h^{t}(s' | s,a) \bar E^t_{h+1}(s', {\pi}^{t+1}(s')) + \|\hat p_h^{t}(\cdot | s,a) -p_h^{t}(\cdot | s,a)\|_1 \sigma_{H-h}  \\
& \leq \sigma_{H-h} \sqrt{2\frac{\beta(n_h^t(s,a),\delta)}{n_h^t(s,a)}} + \sum_{s' \in \cS} \! p_h^{t}(s' | s,a) \bar E^t_{h+1}(s', {\pi}^{t+1}(s'))\,,
\end{align*}}
where we used that $\bar E^t_{h+1}(s', {\pi}^{t+1}(s')) \leq \gamma \sigma_{H-h-1}\leq \sigma_{H-h}$. Plugging this in \eqref{ToPlugIn}, upper bounding $\gamma$ by 1 and using $2\sqrt{2}\leq 3$ yields
\[\overline{E}_h^{t}(s,a) \leq 3 \sigma_{H-h}\sqrt{\frac{\beta(n_h^t(s,a),\delta)}{n_h^t(s,a)}} + \gamma \sum_{s' \in \cS} p_h^{t}(s' | s,a) \bar E^t_{h+1}(s', {\pi}^{t+1}(s'))\]

The conclusion follows by noting that it also holds that \[\overline{E}_h^{t}(s,a) \leq \gamma\sigma_{H-h} \leq 3 \sigma_{H-h} \leq 3\sigma_{H-h} + \gamma\sum_{s' \in \cS}  p_h^{t}(s' | s,a) \bar E^t_{h+1}(s', {\pi}^{t+1}(s'))\]
and that the inequality is also true for $n_h^t(s,a)=0$ with the convention $1/0 = +\infty$.

\subsection{Proof of Lemma~\ref{lem:cnt_pseudo}}

As the event $\cE^{\mathrm{cnt}}$ holds, we know that for all $t < \tau$,
\begin{eqnarray*}n_{\ell}^{t}(s,a) &\geq& \frac{1}{2}\bar n_{\ell}^{t}(s,a) - \beta^{\cnt}(\delta).
\end{eqnarray*}

We now distinguish two cases. First, if $\beta^{\cnt}(\delta) \leq \tfrac{1}{4}\bar n_{\ell}^{t}(s,a)$, then \[\frac {\beta(n_{\ell}^t(s,a), \delta)} {n_{\ell}^t(s,a)}\wedge 1 \leq \frac {\beta(n_{\ell}^t(s,a), \delta)} {n_{\ell}^t(s,a)}\leq \frac {\beta\left(\tfrac{1}{4}\bar n_{\ell}^{t}(s,a), \delta\right)} {\tfrac{1}{4}\bar n_{\ell}^{t}(s,a)} \leq 4   \frac {\beta\left(\bar n_{\ell}^{t}(s,a), \delta\right)} {\bar n_{\ell}^{t}(s,a) \vee 1},\]
where we use that $x \mapsto \beta(x,\delta)/x$ is non-increasing for $x\geq 1$,  $x \mapsto \beta(x,\delta)$ is non-decreasing, and $\beta^{\cnt}(\delta) \geq 1$.

If $\beta^{\cnt}(\delta) > \tfrac{1}{4}\bar n_{\ell}^{t}(s,a)$, simple algebra shows that
\[\frac {\beta(n_{\ell}^t(s,a), \delta)} {n_{\ell}^t(s,a)}\wedge 1\leq   1 < 4 \frac{\beta^{\cnt}(\delta)}{\bar n_{\ell}^{t}(s,a) \vee 1} \leq 4 \frac{\beta(\bar n_{\ell}^{t}(s,a), \delta)}{\bar n_{\ell}^{t}(s,a) \vee 1},\]
where we use that $1 \leq \beta^{\cnt}(\delta)\leq \beta(0,\delta)$ and $x \mapsto \beta(x,\delta)$ is non-decreasing.


 In both cases, we have
\[\left[{ \frac {\beta(n_{\ell}^t(s,a), \delta)} {n_{\ell}^t(s,a)} } \wedge 1\right] \leq 4 {\frac{\beta(\bar n_{\ell}^{t}(s,a), \delta)}{\bar n_{\ell}^{t}(s,a) \vee 1}}.\]

\newpage

\section{An Algorithm for Best Policy Identification} \label{app:BPI}

In this section, we describe the BPI-UCRL algorithm and analyze its sample complexity. BPI-UCRL aims at finding the optimal policy for a fixed reward function $r$, assumed to deterministic. Hence to ease the notation we drop the dependency in $r$ in the value and Q-value functions.

Unlike RF-UCRL, who builds upper bound on the estimation errors, BPI-UCRL relies on confidence intervals on the Q-value function of a policy $\pi$. To define these confidence regions, we first introduce a confidence region on the transition probabilities
\[\cC_h^t(s,a) = \left\{p\in \Sigma_{S} :  \KL\!\big(\hp_h^t(\cdot|s,a),p\big) \leq \frac {\beta(n_h^t(s,a), \delta)} {n_h^t(s,a)}\right\}\]
and define, for each policy $\pi$ the confidence regions after $t$ episodes as
\begin{align*}
  \uQ_h^{t,\pi}(s,a) &= (r_h + \gamma\max_{\up_h\in\cC_h^t(s,a)}\up_h \uV_{h+1}^{t,\pi} ) (s,a) &
    \lQ_h^{t,\pi}(s,a) &= (r_h + \gamma \min_{\lp_h\in\cC_h^t(s,a)}\lp_h \lV_{h+1}^{t,\pi} ) (s,a) \\
  \uV_h^{t,\pi}(s) &= \pi \uQ_h^{t,\pi}(s) &  \lV_h^{t,\pi}(s) &= \pi \lQ_h^{t,\pi}(s)\\
  \uV_{H+1}^{t,\pi}(s) &= 0 &  \lV_{H+1}^{t,\pi}(s) &= 0\\
  \up_h^{t,\pi}(s,a) &\in \argmax{\up\in\cC_h^t(s,a)} \up_h \uV_{h+1}^{t,\pi} (s,a) &
  \lp_h^{t,\pi}(s,a) &\in \argmin{\lp\in\cC_h^t(s,a)} \lp_h \lV_{h+1}^{t,\pi} (s,a)\,,
\end{align*}
where we use the notation $p_hf(s,a) = \EE_{s'\sim p_h(.|s,a)}f(s')$ for the expectation operator and $\pi g(s) = g\big(s,\pi(s)\big)$ for the application of a policy.
We also define upper and lower confidence bounds on the optimal value and Q-value functions as
\begin{align*}
  \uQ_h^{t}(s,a) &= (r_h + \gamma \max_{\up_h\in\cC_h^t(s,a)}\up_h \uV_{h+1}^{t} ) (s,a) &
    \lQ_h^{t}(s,a) &= (r_h + \gamma \min_{\lp_h\in\cC_h^t(s,a)}\lp_h \lV_{h+1}^{t} ) (s,a) \\
  \uV_h^{t}(s) &= \max_{a} \uQ_h^{t}(s,a) &  \lV_h^{t}(s) &= \max_{a} \lQ_h^{t}(s,a)\\
  \uV_{H+1}^{t}(s) &= 0 &  \lV_{H+1}^{t}(s) &= 0\\
  \up_h^{t}(s,a) &\in \argmax{\up\in\cC_h^t(s,a)} \up_h \uV_{h+1}^{t} (s,a) &
  \lp_h^{t}(s,a) &\in \argmin{\lp\in\cC_h^t(s,a)} \lp_h \lV_{h+1}^{t} (s,a)\\
  \upi_h^{t}(s,a) &\in \argmax{a} \uQ_{h}^{t} (s,a) &
  \lpi_h^{t}(s,a) &\in \argmax{a} \lQ_{h}^{t} (s,a)\,.
\end{align*}
By definition of the event $\cE$ in Lemma~\ref{lem:optimism}, note that for all $h$ and $(s,a)$ the true transition probability $p_h(\cdot | s,a)$ belongs to $\cC_h^t(s,a)$ for all $t$, hence one can easily prove by induction that for all $\pi$, $\lQ_h^{t,\pi}(s,a) \leq Q_h^{\pi}(s,a) \leq \uQ_h^{t,\pi}(s,a)$ and $\lQ_h^{t}(s,a) \leq Q_h^{\star}(s,a) \leq \uQ_h^{t}(s,a)$.

\paragraph{BPI-UCRL} We are now ready to present an algorithm for BPI based on the UCRL algorithm, named \OurAlgorithmBPI. It is defined by the three rules:
\begin{itemize}
  \item \textbf{Sampling rule} The policy $\pi^{t+1} = \upi^{t}$ acts greedily with respect to the upper-bounds on the optimal Q-value functions.
  \item \textbf{Stopping rule} $\tau = \inf\big\{ t\in\N : \uV_1^t(s_1)-\lV_1^t(s_1) \leq \epsilon \big\}$\,.
  \item \textbf{Recommendation rule} The prediction $\hpi_\tau = \lpi^\tau$ is the policy that acts greedily with respect to the lower-bounds on the optimal Q-value functions.
\end{itemize}

\OurAlgorithmBPI{} bears some similarities with the online algorithms proposed by \cite{EvenDaral06} to identify the optimal policy in a discounted MDP. They also build (different) upper and lower confidence bounds on $Q^\star(s,a)$ for all $(s,a)$ and output the greedy policy with respect to $\underline{Q}$. However the stopping rule waits until \emph{for all} state $s$ and all action $a$ that has not been eliminated, $|\overline Q (s,a) - \underline{Q}(s,a)| < \varepsilon(1-\gamma) / 2$, which may take a very long time if some states have a small probability to be reached. In our case, only the confidence interval on the optimal value in state $s_1$ needs to be small to trigger stopping. This different stopping rule is also due to a different objective: find an optimal policy in state $s_1$ (or when $s_1$ is drawn from some distribution $P_0$) as opposed to find an optimal policy in all states. In our case, we are able to provide upper bound on the stopping rule, which are not given by \cite{EvenDaral06}, who analyze the sample complexity of a different algorithm that requires a generative model. 

We now given an analysis of \OurAlgorithmBPI{}, which bears strong similarity with the proof of Theorem~\ref{thm:sc} and establishes similar sample complexity guarantees as those proved for \OurAlgorithm{}: a $\cO\left(({SAH^4}/{\epsilon^2})\ln(1/\delta)\right)$ bound in a regime of small $\delta$ and a $\cO\left({S^2AH^4}/{\epsilon^2}\right)$ bound in a regime of small $\varepsilon$. For stationary transitions, the dependency in $H$ in these bounds can be improved to $H^3$, following the same steps as in Appendix~\ref{app:stationary}.

\begin{theorem}\label{thm:sc_BPI} \OurAlgorithmBPI using threshold \[\beta(n,\delta) = \log\big(2SAH/\delta\big) + (S-1)\log\big(e(1+n/(S-1))\big)\]
is $(\varepsilon,\delta)$-correct for best policy identification. Moreover, with probability $1-\delta$, the number of trajectories $\tau$ collected satisfy
{\footnotesize
\begin{align*}
\tau \leq \!\frac{\cC_{H}SA}{\varepsilon^2}
	\!\spa{
		\log\!\pa{\! \frac{2SAH}{\delta}\! }
		\!+ 2(S\!-\!1)\!\log\!\pa{\!
			\frac{\cC_{H}SA}{\varepsilon^2}\!
			\pa{\!
				\log\!\pa{\! \frac{2SAH}{\delta} \!} \! + (S\!-\!1)\pa{\!\sqrt{e} + \sqrt{\frac{e}{S-1}}}
			\!}\!
		} \!+ (S\!-\!1)
	}
\end{align*}
}
where  $\cC_{H} = 64(1+\sqrt{2})^2\sigma_H^4$.
\end{theorem}

\begin{proof}
  The first part of the theorem is a direct consequence of the correctness of the confidence bounds. Indeed on the event $\cE$ if the algorithm stops at time $\tau$, then we know
 \begin{align*}
   V^{\lpi^\tau}_1(s_1) \geq \lV_1^{\tau,\lpi^\tau}(s_1) = \lV_1^\tau(s_1) \geq \uV_1^\tau(s_1) - \epsilon \geq \Vstar_1(s_1) - \epsilon\,.
 \end{align*}
 The fact that the event $\cE$ holds with probability at least $1-\delta$ (see Lemma~\ref{lem:concentration}) allows us to conclude that \OurAlgorithmBPI is $(\varepsilon,\delta)$-correct.

 The proof of upper bounds on the complexity is very close to a classical regret proof. Fix
 some $T<\tau$. Then we know that for all $t\leq T$ it holds
 \[
\epsilon \leq  \uV_1^t(s_1)-\lV_1^t(s_1) \leq \uV_1^t(s_1)-\lV_1^{t,\upi^t}(s_1)\,.
 \]
On the event $\cF$ for a state action $(s,a)$, using the holder inequality, the fact that $\lV_{h+1}^{t,\upi^t}(s')\leq \sigma_{H-h}$ and $\uV_{h+1}^{t}(s')\leq \sigma_{H-h}$ and the Pinsker's inequality we have
\begin{align*}
  \uQ_h^t(s,a)-\lQ^{t,\upi^t}_h(s,a) &= \gamma (\up_h^t-p_h) \uV_{h+1}^t(s,a) + \gamma(p_h-\lp_h^{t,\upi}) \lV_{h+1}^{t,\upi^t}(s,a) + \gamma p_h(\uV_{h+1}^t-\lV^{t,\upi^t}_{h+1})(s,a)\\
   &\leq \| \up_h^{t}(\cdot | s,a) -p_h(\cdot | s,a)\|_1 \sigma_{H-h}+  \| p_h^{t}(\cdot | s,a) -\lp_h^{t,\upi^t}(\cdot | s,a)\|_1 \sigma_{H-h}\\
    &\qquad\qquad+ \gamma p_h(\uV_{h+1}^t-\lV^{t,\upi^t}_{h+1})(s,a)\\
   &\leq  4\sigma_{H-h}\left[\sqrt{\frac{\beta(n_h^t(s,a),\delta)}{n_h^t(s,a)}} \wedge 1 \right] +\gamma p_h(\uV_{h+1}^t-\lV^{t,\upi^t}_{h+1})(s,a)\,.
\end{align*}
Thus, using that $\upi^t = \pi^{t+1}$ and by definition that $\uV_h^{t}(s) = \pi_h^{t+1}\uQ_h^{t}(s)$ and $\lV_h^{t,\upi^t}(s) = \pi_h^{t+1}\uQ_h^{t,\upi^t}(s)$, we obtain a recursive formula for the difference of upper and lower bound on the value functions, which may be viewed as a counterpart of Lemma~\ref{lem:induction} in the proof of Theorem~\ref{thm:sc}:
\[
\uV_h^t(s)-\lV_h^{t,\upi^t}(s) \leq 4\sigma_{H-h}\left[\sqrt{\frac{\beta(n_h^t(s,a),\delta)}{n_h^t(s,a)}} \wedge 1 \right] + \gamma p_h(\uV_{h+1}^t-\lV^{t,\upi^t}_{h+1})(s,a)
\]
Recalling that $p_h^{t}(s,a)$ denotes the probability that the state action pair $(s,a)$ is visited at step $h$ under the policy $\pi^t$ used in the $t$-th episode, we can prove by induction with the previous formula that for all $t < \tau$,
\[
\epsilon \leq \uV_1^t(s)-\lV_1^{t,\upi^t}(s) \leq 4\sum_{h=1}^H \gamma^{h-1}\sigma_{H-h}\sum_{s,a} p_h^{t+1}(s,a) \left[\sqrt{\frac{\beta(n_h^t(s,a),\delta)}{n_h^t(s,a)}} \wedge 1 \right]
\]
Thus summing for all $0\leq t \leq T < \tau$ leads to
\[
(T+1) \epsilon \leq \sum_{t=0}^T 4\sum_{h=1}^H \gamma^{h-1}\sigma_{H-h}\sum_{s,a} p_h^{t+1}(s,a) \left[\sqrt{\frac{\beta(n_h^t(s,a),\delta)}{n_h^t(s,a)}} \wedge 1 \right]\,.
\]
We can then conclude exactly as in the proof of Theorem~\ref{thm:sc}, that is use Lemma~\ref{lem:cnt_pseudo} to relate the counts $n_h^t(s,a)$ to the pseudo-counts $\bar n_h^t(s,a)$ to upper bound the sample complexity on the event $\cF = \cE \cap \cE^{\cnt}$, which holds with probability at least $1-\delta$.
\end{proof}

\section{Analysis in the Stationary Case} \label{app:stationary}

In the stationary case, the transition kernel doesn't depend on time, that is $p_h(s' |s,a) = p(s' |s,a)$ for all $h \in [H]$. In that case, the upper bounds used in the algorithms can take into account the number of visits to $(s,a)$ in any step, $n^t(s,a) = \sum_{h \in [H]} n_h^t(s,a)$. More precisely, in that case, the $\bar E_h^t(s,a)$ are replaced by  the tighter upper bounds
\[
 \tilde{E}_{h}^{t}(s,a)  =  \min \left\{ \gamma\sigma_{H-h} ;  \gamma \sigma_{H-h}\sqrt{\frac{2\beta(n^t(s,a),\delta)}{n^t(s,a)}} + \gamma \sum_{s'} \hat p^t(s' |s,a) \max_{b}\tilde{E}_{h+1}^{t}(s',b)\right\} \]
 where $\hat p^t(s' |s,a) = \frac{\sum_{i=1}^{t}\sum_{h=1}^H \ind\left(s_h^{i} = s, a_h^{i} =a, s_{h+1}^i = s'\right)}{n^t(s,a)}$.

 Letting $\tilde{\cE}$ be the event
 \[\tilde{\cE} = \left(\forall t \in \N^*, \forall (s,a)\in \cS \times \cA, \|\hat p^t(\cdot |s,a) - p(\cdot |s,a)\|_1 \leq \sqrt{\frac{2\beta(n^t(s,a),\delta)}{n^t(s,a)}}\right),\]
 with Pinsker's inequality and Lemma~\ref{lem:max_ineq_categorical}, one can prove that $\bP(\tilde{\cE})\geq 1- \delta/2$ for the choice $\beta(n,\delta) = \log\!\big(2SA/\delta\big) + (S-1)\log\big(e(1+n/(S-1))\big)$. 
 \OurAlgorithm{} based on the alternative bounds $\tilde{E}_h^t(s,a)$ is correct on $\tilde{\cE}$, and we can upper bound its sample complexity on the event $\tilde{\cE} \cap \cE^{\text{cnt}}$ following the same approach as before. Letting
 \[\tilde q_{h}^{t} = \sum_{(s,a)} p_{h}^{t+1}(s,a) \tilde{E}_h^{t}(s,a),\]
 one can establish a similar inductive relationship as that of Lemma~\ref{lem:induction} which yields
 \[\tilde q_{1}^{t} \leq 3 \sum_{h=1}^{H}\sum_{(s,a)} \gamma^{h-1} \sigma_{H-h}p_{h}^{t+1}(s,a)\left[\sqrt{\frac{\beta(n^t(s,a),\delta)}{n^t(s,a)}}\wedge 1\right].\]
 Hence, for every $T < \tau$, as $\tilde q_{1}^{t} \geq \varepsilon/2$ for all $t \leq T$, one can write
\begin{eqnarray*}
\varepsilon(T+1)
& \leq & 6 \sum_{t=0}^{T}\sum_{h=1}^{H}\sum_{(s,a)} \gamma^{h-1} \sigma_{H-h}p_{h}^{t+1}(s,a)\left[\sqrt{\frac{\beta(n^t(s,a),\delta)}{n^t(s,a)}}\wedge 1\right] \\
& \leq & 6 \sigma_H \sum_{(s,a)} \sum_{t=0}^{T}\sum_{h=1}^{H}p_{h}^{t+1}(s,a)\left[\sqrt{\frac{\beta(n^t(s,a),\delta)}{n^t(s,a)}}\wedge 1\right]
\end{eqnarray*}
Letting $\overline{n}^t(s,a) = \sum_{h=1}^{H}  \sum_{i=1}^{t}p_{h}^{i}(s,a)$, observe that $\sum_{h=1}^{H}p_{h}^{t+1}(s,a) = \overline{n}^{t+1}(s,a) - \overline{n}^t(s,a)$ and a similar reasoning than that in the proof of Lemma~\ref{lem:cnt_pseudo} yields
\[\varepsilon (T+1)  \leq  12 \sigma_H \sum_{(s,a)} \sum_{t=0}^{T}\left(\overline{n}^{t+1}(s,a) - \overline{n}^t(s,a)\right)\sqrt{\frac{\beta(\bar n^t(s,a),\delta)}{\bar n^t(s,a) \vee 1}} \]
Using Lemma 19 in \cite{UCRL10} and the Cauchy-Schwarz inequality yields
\begin{eqnarray*}
\varepsilon (T+1) & \leq &  12(1+\sqrt{2}) \sigma_H \sum_{(s,a)} \sqrt{\overline{n}^{T+1}(s,a)} \sqrt{\beta(\bar n^{T+1}(s,a),\delta)} \\
& \leq & 12(1+\sqrt{2}) \sigma_H \sqrt{\beta(H(T+1),\delta)}\sqrt{SA}\sqrt{ \sum_{s,a}\overline{n}^{T+1}(s,a)} \\
& = & 12(1+\sqrt{2}) \sigma_H \sqrt{\beta(H(T+1),\delta)}\sqrt{SA}\sqrt{H(T+1)}.
\end{eqnarray*}
Hence, $\tau$ is finite and satisfies
\[\tau \leq \frac{\cC SAH^3}{\varepsilon^2}\beta(H\tau,\delta)\]
for $\cC = 144(1+\sqrt{2})^2$. 
By Lemma~\ref{lemma:inversion_of_n_inequality}, it follows that {\footnotesize
\begin{align*}
\tau \leq \!\frac{\cC SAH^3}{\varepsilon^2}
	\!\spa{
		\log\!\pa{\! \frac{2SA}{\delta}\! }
		\!+ 2(S\!-\!1)\!\log\!\pa{\!
			\frac{\cC SAH^3\sqrt{H}}{\varepsilon^2}\!
			\pa{\!
				\log\!\pa{\! \frac{2SA}{\delta} \!} \! + (S\!-\!1)\pa{\!\sqrt{e} + \sqrt{\frac{He}{S-1}}}
			\!}\!
		} \!+ (S\!-\!1)
	}\;.
\end{align*}
}

\section{Conversion of UCB-VI to Best Policy Identification}\label{app:ChiJin}

We present here an alternative conversion from UCB-VI to a $(\varepsilon,\delta)$-PAC BPI algorithm, which improves over the one discussed by \cite{Jin18OptQL} and presented in Section~\ref{sec:RF_to_BPI}. 

Let $\cA(\varepsilon)$ be an algorithm that after a deterministic number $K_\varepsilon$ of trajectories starting from $s_0$ outputs a policy $\hat{\pi}$ satisfying 
\[\bP\left(V^{\star}(s_0) - V^{\hat\pi}(s_0) \geq \varepsilon\right) \leq \frac{1}{2}.\]
From \cite{Jin18OptQL}, UCB-VI run for $K_{\varepsilon}= \cO\left(\tfrac{H^2SA}{\varepsilon^2}\right)$ and outputing a policy uniformly at random among the one used satisfies this property. 

Now consider the following algorithm, that depends on three parameters, $\varepsilon_0$, $M$ and $N$: 
\begin{enumerate}
 \item Run $M$ independent instances of $\cA(\varepsilon_0)$ and denote by $\pi_1,\dots,\pi_M$ the $M$ policies returned by these algorithms. 
 \item For each $m\in [M]$, generate $N$ trajectories starting from $s_0$ under the policy $\pi_m$ and define $\hat{V}_m$ to be the average cumulative return of those trajectories. 
 \item Output the policy $\hat{\pi} = \pi_{\hat{m}}$ where $\hat{m} =\argmax{m=1,\dots,M} \hat{V}_m$. 
\end{enumerate}

\begin{proposition}\label{lem:Sufficient} Choosing $\varepsilon_0$, $M$ and $N = \frac{H^2}{2(\varepsilon')^2} \ln\left(\frac{M}{\delta'}\right)$ with 
\[
 \varepsilon_0 + 2 \varepsilon' \leq \varepsilon  \ \ \text{ and } \ \ \ \delta' + \left(\frac{1}{2}\right)^M \leq \delta, 
\]
the above strategy outputs an $\varepsilon$-optimal policy with probability larger than $1 - \delta$. 
\end{proposition}

Clearly, the (deterministic) sample complexity of this algorithm is $MK_{\varepsilon} + M N$, hence with the choices in Proposition~\ref{lem:Sufficient}, the sample complexity becomes 
\[\cO\left(\frac{H^2SA}{\varepsilon^2}\log\left(\frac{1}{\delta}\right) + \frac{H^2}{\varepsilon^2}\log^2\left(\frac{1}{\delta}\right)\right).\]

\begin{proof}
First, it is easy to upper bound the probability that the best of the $M$ policies returned is $\varepsilon_0$ sub-optimal, using that the different instances are independent. 

\begin{lemma} Letting $\tilde{m} = \argmax{m \in [M]} V^{\pi_{m}}(s_0)$, 
we have $\bP\left(V^\star(s_0) - V^{\pi_{\tilde m}}(s_0) \geq \varepsilon_0\right) \leq \left(\tfrac{1}{2}\right)^{M}$. 
\end{lemma}

Of course, the algorithm does not know the values exactly and does not have access to $\pi_{\tilde{m}}$. However, the estimated values are not too far from these values if $N$ is chosen carefully. More precisely, Hoeffding's inequality and a union bound tell us the following. 

\begin{lemma}\label{lem:EstimValues} If $N = \frac{H^2}{2(\varepsilon')^2} \ln\left(\frac{M}{\delta'}\right)$ then $\bP\left(\exists m \in [M]: |\hat{V}_m - V^{\hat{\pi}_m}(s_0)| \geq \varepsilon'\right)\leq \delta'$. 
\end{lemma}

Using the above lemmas, the event 
\[\cE = \left(\forall m \in [M], |\hat{V}_m - V^{\hat{\pi}_m}(s_0)| \leq \varepsilon'\right) \cap \left(V^\star(s_0) - V^{\pi_{\tilde m}}(s_0) \leq \varepsilon_0\right)\]
is of probability at least $1 - \delta' - \left(\tfrac{1}{2}\right)^M$. Moreover, on the event $\cE$, letting 
$\hat{m} = \argmax{m \in M} \hat{V}_m$ and $\hat{\pi} = \pi_{\hat{m}}$ the policy returned by algorithm, it holds that 
\begin{eqnarray*}
 V^\star(s_0) - V^{\hat{\pi}}(s_0) & = & V^\star(s_0) - V^{\pi_{\tilde m}}(s_0) + V^{\pi_{\tilde m}}(s_0) - \hat V_{\tilde{m}} + \underbrace{\hat{V}_{\tilde{m}} - \hat{V}_{\hat{m}}}_{\leq 0} + \hat{V}_{\hat{m}} - {V}^{\pi_{\hat{m}}}(s_0) \\
 & \leq & \varepsilon_0 + 2 \varepsilon'.
\end{eqnarray*}
The conditions stated in Proposition~\ref{lem:Sufficient} guarantee that $\bP(\cE) \geq 1 - \delta$ and that on $\cE$, $ V^\star(s_0) - V^{\hat{\pi}}(s_0) \leq \varepsilon$, which proves Proposition~\ref{lem:Sufficient}.
\end{proof}

\section{Technical Lemmas}\label{app:technical}

\begin{lemma}
	\label{lemma:inversion_of_n_inequality}
	Let $n \geq 1$ and $a, b, c, d > 0$. If $n \Delta^2 \leq a + b \log( c+ dn)$
	then
	\begin{align*}
	n \leq \frac{1}{\Delta^2}\left[ a + b \log\pa{c + \frac{d}{\Delta^4} (a + b(\sqrt{c} + \sqrt{d}))^2 } \right].
	\end{align*}
\end{lemma}
\begin{proof}
	Since $\log(x) \leq \sqrt{x}$ and $\sqrt{x+y} \leq \sqrt{x}+\sqrt{y}$ for all $x, y > 0$, we have
	\begin{align*}
	& n\Delta^2 \leq a + b \sqrt{c+dn} \leq a + b\sqrt{c} + b\sqrt{d} \sqrt{n} \\
	& \implies \sqrt{n} \Delta^2  \leq \frac{ a + b\sqrt{c}}{\sqrt{n}} +  b\sqrt{d}  \leq a + b(\sqrt{c} + \sqrt{d}) \\
	& \implies n \leq \frac{1}{\Delta^4}\pa{a + b(\sqrt{c} + \sqrt{d})}^2.
	\end{align*}
	Hence,
	\begin{align*}
	& n \Delta^2 \leq a + b \log( c+ dn) \\
	& \implies n \Delta^2 \leq a + b \log( c+ dn)
	\quad \text{ and } \quad n \leq \frac{1}{\Delta^4}\pa{a + b(\sqrt{c} + \sqrt{d})}^2 \\
	& \implies  n \Delta^2 \leq a + b \log\pa{ c+  \frac{d}{\Delta^4}\pa{a + b(\sqrt{c} + \sqrt{d})}^2 }.
	\end{align*}
\end{proof}

We recall for completeness Lemma 19 in \cite{UCRL10} which is used in our sample complexity analysis. 

\begin{lemma}[Lemma 19 in \cite{UCRL10}] For any sequence of numbers $z_1,\dots,z_n$ with $0 \leq z_k \leq  Z_{k-1} = \max \left[1 ; \sum_{i=1}^{k-1} z_i \right]$
\[\sum_{k=1}^{n} \frac{z_k}{\sqrt{Z_{k-1}}} \leq \left(1+\sqrt{2}\right)\sqrt{Z_n}\;.\]
\end{lemma}

\end{document}